\documentclass[lettersize,journal, twopage]{IEEEtran}

\usepackage{algorithm}
\usepackage[noend]{algpseudocode}
\algrenewcommand\alglinenumber[1]{\scriptsize #1:}

\usepackage{mathtools}   
\usepackage{amsthm}
\usepackage{amsfonts}
\usepackage{amssymb}
\usepackage{bbm}
\usepackage{dsfont}
\usepackage{stmaryrd}

\newtheorem{theorem}{Theorem}
\newtheorem{lemma}{Lemma}

\newtheorem{corollary}{Corollary}
\newtheorem{assumption}{Assumption}
\theoremstyle{definition}

\theoremstyle{remark}
\newtheorem{remark}{Remark}

\usepackage{graphicx}
\usepackage{array}
\usepackage{float}
\usepackage{multirow}
\usepackage{booktabs}
\usepackage[font=footnotesize]{subcaption}

\usepackage{tikz}
\usetikzlibrary{positioning,arrows.meta,calc,shapes,decorations.pathreplacing}
\usepackage{pgfplots}
\pgfplotsset{compat=1.17}

\usepackage{textcomp}
\usepackage{stfloats}
\usepackage{url}
\usepackage{cite}
\usepackage{enumitem}
\usepackage{mdframed}
\usepackage{multicol}

\usepackage[switch]{lineno}
\usepackage{xcolor}

\usepackage{etoolbox}
\AtBeginEnvironment{figure}{\nolinenumbers}
\AtEndEnvironment{figure}{\linenumbers}
\AtBeginEnvironment{table}{\nolinenumbers}
\AtEndEnvironment{table}{\linenumbers}

\usepackage{hyperref}

\newcommand{\cvar}{\operatorname{CVaR}}
\DeclareMathOperator*{\esssup}{ess\,sup}

\newcommand{\fag}{f_{\alpha,\gamma}}
\newcommand{\bestmean}[1]{\textbf{#1}}
\newcommand{\bestband}[1]{\textcolor{blue}{#1}}

\allowdisplaybreaks
\begin{document}

\title{\textsc{FedeRage}: Provably Convergent Agnostic \\Federated Learning under General Client Drift}

\author{Herlock~(SeyedAbolfazl)~Rahimi and Dionysis~Kalogerias,~\textit{Senior Member, IEEE\vspace{-14pt}}%

\thanks{This work has been supported by the US National Science Foundation under Grants 2242215 and 2431860.}%
\thanks{Preliminary results leading to this paper were presented in part at the \textit{2025 IEEE International Workshop on Computational Advances in Multi-Sensor Adaptive Processing (CAMSAP)}~\cite{RahimiKalogerias2025Agnostic}.}}

\markboth{}{}

\maketitle

\begin{abstract}
Federated learning (FL) enables collaborative model training without sharing raw data, but its performance degrades under non-IID data and stochastic client participation. Remedies built on classical Federated Averaging (FedAvg) typically presuppose that client participation probabilities are known to the server, which is rarely the case in deployed systems. We first discuss and then characterize the optimization problem that \emph{distributionally agnostic} FedAvg actually solves when participation is entirely unknown, possibly highly skewed, and of variable size across rounds: uniform aggregation is shown to minimize a well-defined stochastic objective, weighted by the participation-induced marginal, at a standard $\mathcal{O}(1/\sqrt{T})$ rate for convex and possibly nonsmooth losses. Building on this characterization, we propose \emph{Federated Risk-Averse Averaging} (\textsc{FedeRage}), a risk-averse extension of FedAvg that embeds the \emph{Conditional Value-at-Risk} (CVaR) into the local objective within a natural distributionally robust optimization (DRO) framework. \textsc{FedeRage} implicitly upweights high-loss and infrequently participating clients while adding only a \emph{single scalar per-client}, and admits an $\mathcal{O}(\kappa/\sqrt{T})$ rate in which the factor $\kappa$ is the upper bound on the ``price" of risk aversion. In contrast with aggregation-alignment schemes based on optimal transport, which require the availability distribution as an input, \textsc{FedeRage} remains agnostic to it. Several experiments on three heterogeneous benchmarks indicate consistent improvements over state-of-the-art methods in accuracy, fairness, and convergence speed.
\end{abstract}
\vspace{-4pt}
\begin{IEEEkeywords}
Federated learning, distributionally robust optimization, conditional value-at-risk, risk-averse optimization, client heterogeneity, partial participation, convergence analysis.
\end{IEEEkeywords}

\vspace{-4pt}
\section{Introduction}\label{Introduction}

\IEEEPARstart{F}{ederated} learning (FL) allows a population of distributed clients to collaboratively train a shared model without exchanging raw data, thereby preserving privacy and reducing communication overhead~\cite{Konecny2017FL,FedAvg,Kairouz2021,Geyer2017,Bagdasaryan2018,Reddi2016,Coppola2015}. In synchronous FL, each communication round proceeds as follows: a central server broadcasts the current global model to all clients; each client updates the received model locally on its private dataset (the \emph{local rounds}); and a subset of active clients then transmit their updated parameters back to the server. The server aggregates the received parameters---most commonly through \emph{Federated Averaging} (FedAvg)~\cite{FedAvg,Konecny2015Opt}---and the process repeats until convergence (see Algorithm~\ref{alg:fedavg}). The efficiency, stability, and fairness of this decentralized procedure depend not only on the local data distributions but also on the stochastic availability of the clients.

Performance degradation in FL systems can be traced to two principal sources of heterogeneity. The first is \emph{statistical (data) heterogeneity}: clients often possess non-IID data due to personalized usage patterns, device contexts, or geographical factors~\cite{FedNonIID,LeCun1998,Lin2017}. This induces \emph{client drift}, the divergence of local optimization trajectories, which slows or destabilizes aggregation and degrades generalization; this has been studied extensively in the FL literature~\cite{Li2019,Pillutla_2023,Hitaj2017,SCAFFOLD,FedNova,FedDisco}. The second source is \emph{system and participation heterogeneity}: clients differ in computational capacity, communication bandwidth, and, most critically, \emph{availability}. Unlike classical distributed optimization with deterministic scheduling, client participation in FL is inherently stochastic, constrained by factors such as network connectivity, battery state, and user activity~\cite{Hong2018,Jakovetic2013,Li2014a,ChoWangJoshi2022,FieldGuideFedOpt,SemiCyclicSGD,FlexibleFLParticipation,Ribero2023}. Such \emph{restricted availability} systematically biases aggregation toward frequently active clients while underrepresenting sporadic participants, resulting in slower convergence and degraded fairness~\cite{theodoropoulos2023ram}.
\begin{figure}[t]
    \centering
    \begin{subfigure}[b]{0.51\linewidth}
        \centering
        \includegraphics[width=\linewidth]{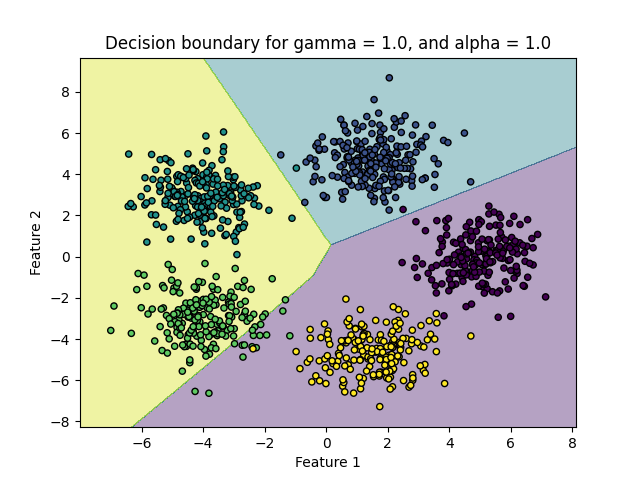}
        \label{fig:subfig1}
    \end{subfigure}%
    \hspace{-12pt}
    \begin{subfigure}[b]{0.51\linewidth}
        \centering
        \includegraphics[width=\linewidth]{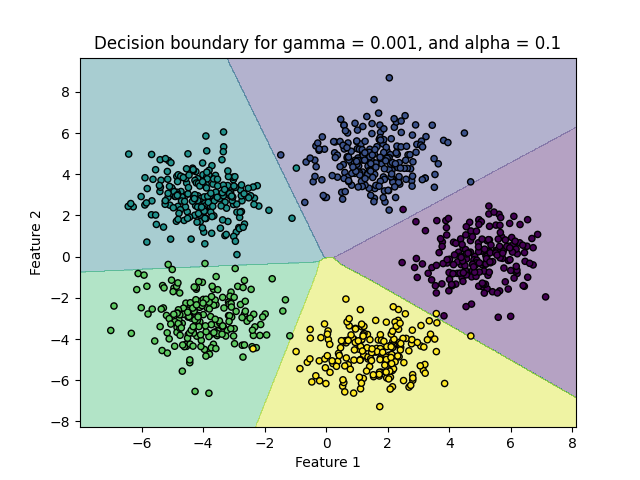}
        \label{fig:subfig2}
    \end{subfigure}
    \vspace{-8pt}
    \begin{subfigure}[b]{0.95\linewidth}
        \centering
        \includegraphics[width=\linewidth]{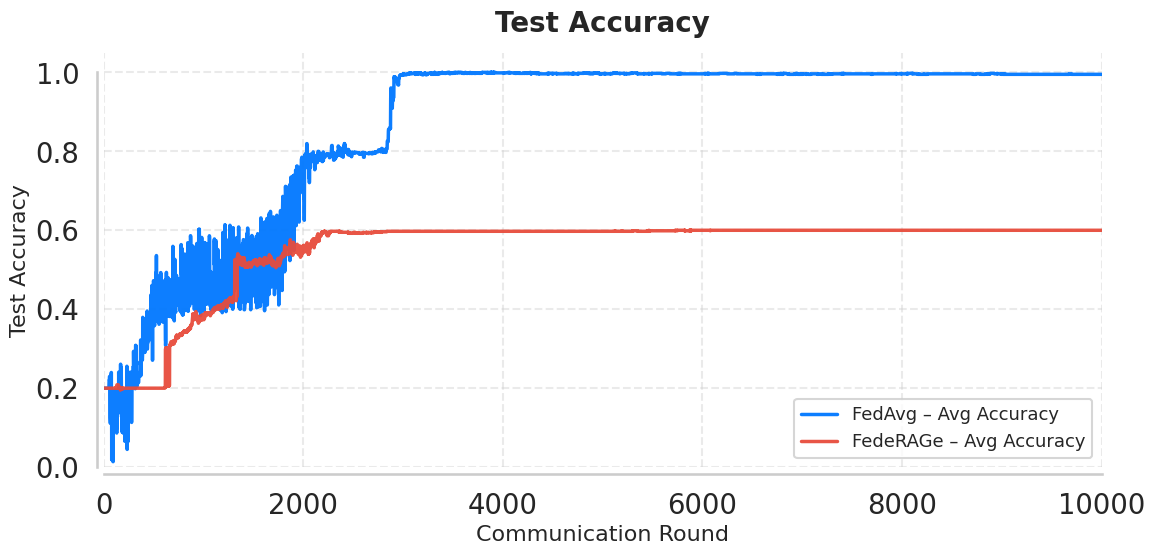}
        \label{fig:subfig3}
    \end{subfigure}
    \caption{{Motivating example: Interaction of statistical heterogeneity and skewed availability.}
    Five clients possess highly non-IID data (clustered by user). Client availability is skewed: three clients participate with probability $0.3$, two with $0.05$.
    \textbf{Top:} learned decision regions after training. \textbf{Bottom:} test accuracy versus communication rounds.
    FedAvg overfits to frequently available clients, degrading accuracy for infrequent ones and slowing convergence.
    \textit{FedeRage} (proposed) mitigates this bias, converging faster and attaining higher final accuracy under the same non-IID and availability regime.}
    \label{fig:motivation}
    \vspace{-14pt}
\end{figure}
A key observation underlying this work is that, beyond being non-uniform and potentially highly skewed, client availability is in practice also \textit{unknown to the server}, who nonetheless must coordinate aggregation in standard FL protocols~\cite{FedNonIID}. This distinction is consequential: most existing analyses either assume known participation probabilities or restrict attention to simplified availability models that do not capture this uncertainty. Following~\cite{theodoropoulos2023ram}, we adopt a \emph{Random Access Model} (RAM) to formalize the phenomenon. In the RAM setting, an independent random mechanism operates at each communication round: it receives all clients' local updates and returns a subset of participating clients according to a fixed but unknown distribution, inaccessible to both the clients and the server. This abstraction captures the stochastic, decentralized nature of participation in practical FL systems while cleanly decoupling availability from optimization dynamics.


The two axes above are tightly coupled: \emph{non-IID data amplify the effect of unequal participation, and participation skew exacerbates non-IID bias}. Although FedAvg converges under suitable smoothness and sampling conditions~\cite{FedAvg,Khaled2019}, its performance deteriorates in regimes exhibiting both data heterogeneity and availability imbalance. Fig.~\ref{fig:motivation} illustrates this interaction in a small FL instance with five user-specific data clusters. Under strongly non-IID data and skewed availability, FedAvg biases its decision boundary toward frequently participating clients, leaving rare users underfit. The framework proposed here, \textsc{FedeRage} (Section~\ref{ProposedApproach}), rebalances the effective weighting of clients in light of stochastic participation, yielding both faster convergence and more equitable performance.

\vspace{-5pt}
\subsection{Prior Work}\label{sec:prior}

Methods that mitigate heterogeneity in FL fall broadly into two families: \emph{client-side optimization}, which modifies the local objectives or updates so as to align local and global descent directions, and \emph{server-side aggregation}, which reweights client updates to reflect heterogeneity in effort, data, or availability. The two are complementary. The method proposed here belongs to the first family, and we benchmark it against the state of the art of that type.

\noindent\emph{Client-side optimization.} \textsc{FedProx}~\cite{FedProx} stabilizes training with a proximal term that keeps local updates close to the global model, and SCAFFOLD~\cite{SCAFFOLD} employs control variates to correct client drift, at the cost of communicating an additional model-sized vector per round. Complementary analyses~\cite{FedNonIID,Khaled2019} establish convergence of local-update methods under random or partial participation. These approaches curb drift, but treat participation as either full or drawn from a \emph{known} distribution.

\noindent\emph{Server-side aggregation.} \textsc{FedOpt} and its variants~\cite{ReddiAdaptiveFO} replace the plain server average with adaptive, momentum-based server updates; FedNova~\cite{FedNova} normalizes contributions by local training effort; and anarchic FL~\cite{AnarchicFL} adapts aggregation to asynchronous participation. FedDisco~\cite{FedDisco} reweights clients according to label imbalance, while availability-aware client selection~\cite{Ribero2023,ChoWangJoshi2022} determines which clients to solicit under intermittent participation. These methods improve robustness, but the reweighting they apply presupposes knowledge of, or the ability to actively probe, the participation statistics.

\noindent\emph{Availability versus importance distributions.} A more recent thread makes explicit the distinction between the \emph{availability} distribution, which governs how often each client reaches the server, and the \emph{importance} distribution, which encodes the weighting the designer intends clients to receive in the global objective, whether for fairness, robustness, or operational reasons. When the two are misaligned, plain FedAvg converges to a \emph{surrogate} objective weighted by the availability-induced marginal instead of to the intended one~\cite{FedNonIID}. Transport-based aggregation removes this mismatch exactly, by casting aggregation as a masked optimal transport problem between the two marginals~\cite{RahimiKalogerias2026FedAVOT}; importance-sampling and client-selection schemes pursue the same objective by other means~\cite{Pillutla_2023,Nguyen2022}. Each of these mechanisms, however, requires the server to know---or to be able to estimate---the availability distribution, which enters the transport or importance-sampling problem as a hard marginal constraint. The regime studied here is the complementary one, in which that distribution is genuinely inaccessible; the response proposed below is a risk-averse local objective that hedges over a neighborhood of the induced distribution instead of aligning to a prescribed target.

\noindent\emph{A note on terminology.} The word \emph{agnostic} carries a different meaning in the \emph{agnostic federated learning} of Mohri \emph{et al.}~\cite{MohriAFL}, where the global model is trained against the worst-case mixture of client distributions, i.e., a minimax problem over a simplex of \emph{target} weights. That formulation is agnostic to the \emph{deployment} distribution, but presumes that the server can solicit and weight every client during training. Throughout this paper, ``agnostic'' instead denotes the \emph{algorithm's} ignorance of the \emph{participation} law and that the server never observes, estimates, or uses $\mathcal{R}$, and simply averages whatever it receives. 

\vspace{-5.1pt}
\subsection{Motivation for \textsc{FedeRage}}
\vspace{-1pt}
Most existing techniques treat data heterogeneity and stochastic participation as independent issues, whereas the two are coupled in practice: rarely available clients often hold rare or distinctive data. To address this joint difficulty, we introduce \emph{Federated Risk-Averse Averaging} (\textsc{FedeRage}), an FL algorithm built on a risk-averse framework that embeds the \emph{Conditional Value-at-Risk} (CVaR) into the local objective. Coupling the participation distribution with a mean--CVaR formulation causes \textsc{FedeRage} to upweight high-loss and infrequently selected clients implicitly, conferring robustness to statistical and availability heterogeneity within a single DRO framework. The algorithm (Algorithm~\ref{alg:federage}) retains the structure and per-round cost of FedAvg while admitting a convergence guarantee and improved operational fairness. A CVaR-based scheme of similar flavor was proposed in~\cite{theodoropoulos2023ram} and likewise rebalances updates toward underrepresented clients, but under a more restrictive availability model and without convergence guarantees. Related Wasserstein- and CVaR-based DRO formulations~\cite{Pillutla_2023,Nguyen2022,Yu2023,Deng2020,Hong2021,Shi2023} improve robustness and fairness, but generally assume known participation probabilities or incur high computational cost. Stochastic first-order methods for CVaR and related DRO objectives are well understood in the \emph{centralized} setting~\cite{LevyCarmonDuchiSidford2020}; the difficulty specific to the present setting is that the nominal distribution is itself generated by an unobserved participation mechanism, so neither the reference measure nor its likelihood ratios are available to the algorithm.

\vspace{1pt}
\noindent\textbf{\textit{Contributions---}}Those are as follows:
\vspace{-1pt}
\begin{itemize}[leftmargin=.35cm]
    \item \emph{General participation model.} We formulate a stochastic model of client availability that extends the RAM of~\cite{theodoropoulos2023ram} to \emph{variable-size} multi-client participation per round, inducing non-uniform and hidden sampling probabilities consistent with the behavior of deployed FL systems.
    \item \emph{Convergence of agnostic FedAvg.} We identify the objective that agnostic averaging actually minimizes under this model, and later establish an optimal $\mathcal{O}(1/\sqrt{T})$ rate for convex, possibly nonsmooth losses under an unknown, non-uniform, variable-size participation law.
    \item \emph{Risk-aware formulation.} We propose \textsc{FedeRage}, a CVaR-based modification of the local objective that treats data and availability heterogeneity within a single DRO framework, at the cost of a \emph{single additional scalar} of per-client state.
    \item \emph{Convergence of \textsc{FedeRage}.} We prove that the agnostic FedAvg analysis transfers to the mean--CVaR objective, with the rate inflated by the explicit factor $\kappa(\alpha,\gamma)=(1-\gamma)+\gamma/\alpha$ where $\gamma \in [0,1]$ interpolates between mean ($\gamma=0$) and CVaR ($\gamma=1$), 
and $\alpha \in (0,1]$ is the CVaR quantile level.
    \item \emph{Empirical validation.} We show that \textsc{FedeRage} improves on state-of-the-art algorithms in both heterogeneous and restricted-availability regimes, in convergence rate, accuracy, and per-client fairness, with a margin that widens as the learning task becomes harder.
\end{itemize}

\noindent\textbf{Paper organization.}
Section~\ref{ProblemDefinition} introduces the preliminaries and the problem setup. Section~\ref{ProposedApproach} presents the proposed \textsc{FedeRage} framework and its primal and dual interpretations. Section~\ref{Convergence} develops the convergence analysis, with all proofs given in place. Section~\ref{sec:experiments} reports the experiments and implementation details. Section~\ref{Conclusion} concludes.

\begin{algorithm}[t]
\caption{Agnostic Federated Averaging (FedAvg)}
\label{alg:fedavg}
\begin{algorithmic}[1]
\State \textbf{Initialize:} $\theta_i^{0}=\mathbf{0}$ for $i\in[N]$; horizon $T$; local steps $H$; step size $\eta_\theta>0$; constraint set $\mathcal{C}\subseteq\Theta$.
\For{$t = 1,2,\dots, TH$}
  \If{$t \bmod H = 0$} \Comment{\emph{\textbf{Global communication}}}
    \State Clients transmit local parameters through RAM.
    \State Server receives $S^t \subseteq [N]$.
    \State Server \textbf{aggregates}:
    $\displaystyle{\hat{\theta}^t = \frac{1}{|S^t|} \sum_{i \in S^t} \theta_i^{t-1}}$.
    \State Server \textbf{broadcasts} $\hat{\theta}^t$ and sets $\theta_i^t = \hat{\theta}^t$, $\forall i \in [N]$.
  \Else \Comment{\emph{\textbf{Local updates}}}
    \State Each client performs  stochastic subgradient step:
    \State $\theta_i^t
      = \Pi_{\mathcal{C}}\!\left(
        \theta_i^{t-1}
        - \eta_\theta \nabla_\theta f_i(\theta_i^{t-1}; \xi_i^t)
      \right),\ \forall i \in [N]$.
  \EndIf
\EndFor
\State \textbf{Return:} $\dfrac{1}{T} \sum_{\tau=1}^{T} \hat{\theta}^{\tau H}$.
\end{algorithmic}
\end{algorithm}

\section{Problem Definition}\label{ProblemDefinition}
We tacitly consider a distributed multi-class classification task with feature space \(\mathfrak{X} \subset \mathbb{R}^d \), target classes \( \mathcal{C}=\{1,..., C\} \), parameter space $ \Theta\subseteq\mathbb{R}^{d'}$, a given loss function $\ell:\mathcal{C}\times\mathcal{C}\rightarrow\mathbb{R}_+$, and a parametric model (i.e., a learning representation) $m:\mathfrak{X}\times\Theta\rightarrow \mathcal{C}$ \footnote{While we work in the classification setting, our considerations work more generally, e.g,  for regression tasks as well.}. Unlike traditional (centralized) learning, in \textit{Federated Learning (FL)}, data are distributed among $N$ clients (or agents, or users), each with their own private dataset $D_i = \{(X_i^1,Y_i^{1}), \dots, (X_i^{n_i},Y_i^{n_i})\}$ and local parameter $\theta_i \in \Theta$, as well as their own local expected loss \begin{align}
    \hspace{-2pt}f(\theta; D_i) & = \mathbb{E}_{ \mathcal{D}_i}[\ell(m(X, \theta), Y)] = \frac{1}{n_i}\sum_{j=1}^{n_i}\ell(m(X^j_i, \theta), Y^j_i),
\end{align} where $\mathcal{D}_i$ denotes the corresponding empirical distribution induced by the local dataset $D_i$, for each client $i$. Throughout the paper we will use the following conventions as well: $f_i(\cdot)=f(\cdot, D_i), f(\cdot) = \sum_{i=1}^n p_if_i(\cdot)$.

As mentioned earlier, due to different potential scenarios, users transmit their parameters with different frequencies to the server. To formally model this phenomenon, we consider a Random Access Model (RAM) (initially formalized in \cite{theodoropoulos2023ram} and substantially extended herein), that selects a subset of users at each round and transmits their parameters to the server with a fixed but unknown probability distribution (see Section \ref{CanonicalRAM}).

The generic iterated agnostic coordination scheme under which we consider the FL problem throughout this paper is described as follows: At each (global) communication round, all $N$ users (attempt to) transmit their local parameter vectors to the server (passing through the RAM), and the RAM selects $M$ out of $N$ users based on a certain but general probabilistic structure {(see Section \ref{CanonicalRAM} for details)}. Then, the server aggregates the received parameter vectors ($M$ in number) with an aggregation policy and broadcasts the aggregation result to all $N$ clients (for them to subsequently process their local parameters). 
Under this setting, a conventional goal of the server (as in standard FL) is to find a global parameter $\theta^*$ that performs optimally on the weighted loss of all users, i.e., to solve the problem
\begin{align}\label{EmpricialRiskNeutralFL}
    \inf_{\theta \in \Theta}  \hspace{-1bp} 
    \Bigg\{ \hspace{-2bp}
    \sum_{i=1}^{N}p_i f(\theta; D_i)
    \hspace{-2bp}=\hspace{-1bp}
    \mathbb{E}_{I\sim\mathcal{P}}\big[f(\theta; D_I)\big]
    \hspace{-2bp}\Bigg\},
\end{align}
where the weights $\{p_i\}^N_{i=1}$ (defining $\mathcal{P}$) constitute appropriate \textit{posterior} client participation probabilities \textit{induced} by the stochastic mechanism implemented by the RAM, reflecting the \textit{systemic randomness} of user availability (see Section \ref{CanonicalRAM}).

Problem \ref{EmpricialRiskNeutralFL} could (hopefully) be tackled by means of the (\textit{agnostic}) \textit{Federated Average Algorithm (FedAvg)} (see Algorithm \ref{alg:fedavg}), consisting of two alternating stages, also called \textit{rounds}: \textit{local update} rounds, and \textit{global communication} rounds, closely resembling the generic coordination scheme outlined above. More specifically, during each local update round, all users optimize their respective parameters via mini-batch Stochastic Gradient Descent (SGD) run for $H$ consecutive iterations\footnote{For simplicity, we assume that all users employ the same number of local iterations $H$; this  assumption can be relaxed in both analysis and practice.}. During each global communication round, the optimized parameters of all $N$ users are \textit{transmitted towards the server, first passing through the RAM} (which may also be thought of as an multi-erasure channel; see Section \ref{CanonicalRAM}). Then, the server aggregates the received parameters of the resulting $M\le N$ users surviving the RAM by taking a simple (agnostic) average and subsequently broadcasts the result to all users, so that they can (re)-compute their local updates initialized at the new global parameter.

At this point, it is worth mentioning that whether (agnostic) FedAvg (Algorithm \ref{alg:fedavg}) is in fact an appropriate method to solve Problem \ref{EmpricialRiskNeutralFL} remains an open question in the current literature. The agnostic nature of the algorithm ---embodied in its use of uniform averaging over the $M$ received parameters from the clients--- may not necessarily account for the client availability distribution $\mathcal{P}$ induced by the RAM. Since this distribution is hidden from the server, the naive (though straightforward) averaging step in Line 4 of Algorithm \ref{alg:fedavg} may misrepresent the actual statistical significance or ``representativeness" of the contribution of each client to the global model.

This issue raises fundamental concerns about the compatibility of FedAvg with the objective in Problem \ref{EmpricialRiskNeutralFL}, particularly when the sampling process exhibits persistent heterogeneity. A rigorous treatment of this mismatch, including a detailed convergence analysis of FedAvg under a canonical probabilistic RAM structure and for a naturally chosen $\mathcal{P}$ is developed in Section \ref{CanonicalRAM} and later in Section \ref{Convergence}.
\vspace{-6pt}
\subsection{An Alternative Representation of the FL Problem}

Suppose that all clients adopt a common \emph{mini-batch size} $b$. Each mini-batch is a set $\xi_i^j$, with $i \in [N]$ and $j \in \{1, \dots, N^b_i\}$, where $N^b_i:=\binom{n_i}{b}$; let $D^b_i= \bigcup_{j=1}^{N^b_i}\{\xi_i^j\}$ be the collection of all mini-batches of size $b$. Writing the average loss over a mini-batch $\xi$ as $f(\theta; \xi) = \frac{1}{b} \sum_{(X,Y) \in \xi} \ell(m(X, \theta), Y)$ and the average over all mini-batches as $f(\theta; D_i^b) = (1/N^b_i)\sum_{j=1}^{N^b_i}f(\theta; \xi_i^j)$, a double-counting argument yields
\begin{equation}\label{equation:equivalence}
\begin{aligned}
    f(\theta; D_i^b)
    &=\frac{1}{N^b_i\,b}\sum_{j=1}^{N^b_i}\sum_{(X,Y) \in \xi_i^j} \ell(m(X, \theta), Y) \\
    &= \frac{1}{N^b_i\,b} \sum_{(X,Y)\in D_i} \binom{n_i-1}{b-1}\, \ell(m(X, \theta), Y) \\
    &= \frac{1}{n_i}\sum_{(X,Y)\in D_i}\ell(m(X, \theta), Y)
    =f(\theta; D_i),
\end{aligned}
\end{equation}
where we used $\binom{n_i-1}{b-1}/\big(N_i^b\, b\big)=1/n_i$. Hence, for every mini-batch size $b$, Problem~\eqref{EmpricialRiskNeutralFL} can be re-expressed as
\begin{equation}\label{eq:EmpiricalRiskNeturalBatchFL}
    \inf_\theta\Bigg\{\hspace{-1bp}\sum_{i=1}^N p_i f(\theta;D_i^b) \hspace{-2bp}=\hspace{-1bp}
    \mathbb{E}_{I\sim\mathcal{P}}\big[f(\theta; D_I^b)\big]
    \hspace{-2bp}\Bigg\},
\end{equation}
or, more compactly,
\begin{equation}\label{BatchFL}
    \boxed{\inf_\theta \,\mathbb{E}_{\xi\sim \mathcal{Q}^b}\left[f(\theta; \xi)\right],}
\end{equation}
where $\mathcal{Q}^b$ denotes the \emph{mixture distribution} $\mathcal{Q}^b=\sum_{i=1}^N p_i\,\mathcal{Q}_i^b$, with $\mathcal{Q}_i^b$ the uniform distribution over $D_i^b$ and $\mathcal{P}$ the participation distribution. Under $\mathcal{Q}^b$, the probability that $\xi$ equals $\xi_i^j$ is $p_i/N^b_i$ for $i \in [N]$ and $j \in \{1, \dots, N^b_i\}$, and zero otherwise. We write $\xi_i$ for the random variable taking values $\xi_i^j$, $j\in[N_i^b]$, under $\mathcal{Q}_i^b$.

Problems~\eqref{EmpricialRiskNeutralFL}, \eqref{eq:EmpiricalRiskNeturalBatchFL}, and~\eqref{BatchFL} share the same objective. In~\eqref{eq:EmpiricalRiskNeturalBatchFL} and~\eqref{BatchFL}, however, the mini-batch size $b$ appears as part of the problem \emph{formulation} instead of as a hyperparameter of a particular solver (e.g. SGD). We therefore refer to $b$ as the \emph{model mini-batch size}, to distinguish it from the \emph{algorithmic mini-batch size} used within an SGD implementation. Operationally, solving~\eqref{EmpricialRiskNeutralFL} by SGD with algorithmic mini-batch $b$ is equivalent to solving~\eqref{eq:EmpiricalRiskNeturalBatchFL} or~\eqref{BatchFL} by the same scheme with algorithmic mini-batch $1$.

\begin{remark}
While the minibatch reformulation \eqref{BatchFL} does not seem to offer an apparent operational advantage concerning standard FL, the identification of the minibatch distribution $\mathcal{Q}^b$ will be key to the development of our proposed approach for simultaneously dealing for the issues of data heterogeneity and limited client availability (i.e., general client drifts), following the general paradigm of \textit{distributionally robust optimization (DRO)}, as discussed in later Sections \ref{DRFL} and \ref{ProposedApproach}.
\end{remark}

\vspace{-6pt}
\subsection{Canonical Probabilistic RAM}\label{CanonicalRAM}

FL is traditionally analyzed under the assumption that either all clients (full participation) or a fixed fraction of them (partial participation) participate in a round, the shortfall being attributed to device unavailability, battery constraints, or network failures~\cite{FlexibleFLParticipation,SemiCyclicSGD,FieldGuideFedOpt}. In the simplest case the server is assumed to know the participation probabilities and to exploit them during aggregation, which, as discussed above, is impractical. To model participation without such knowledge, we adopt and extend the \emph{Random Access Model} (RAM) of~\cite{theodoropoulos2023ram}: an external, independent entity selects a subset of clients $S^t \subseteq [N]$ at each communication round $t$.

\vspace{4pt}
\noindent\textbf{Canonical fixed-size RAM.}
Let $M \in \{1, \ldots, N\}$ be fixed. At each round $t$, the RAM selects $M$ clients \emph{without replacement} according to a stationary, possibly non-uniform selection prior $\mathcal{R}$; that is, $S^t$ is drawn so that client $i$ satisfies $\{i \in S^t\}$ with marginal probability $\mathbb{P}(i \in S^t)$ determined by $\mathcal{R}$.\footnote{If $M=N$, then regardless of $\mathcal{R}$ the marginal is uniform, $\mathbb{P}(i \in S^t)=1$, and the induced weights below satisfy $p_i=1/N$.} Since all $M$ selected clients reach the server and are weighted uniformly, the weight effectively borne by client $i$'s parameter is proportional to $\mathbb{P}(i\in S^t)/M$. Accordingly, the natural weight on client $i$'s empirical risk $f_i(\theta)$ in the global objective is its \emph{probability of survival} through the two-stage process of RAM selection followed by uniform averaging, i.e., $\mathbb{P}(i \in S^t)/M$: a client's influence on the global model is proportional to its frequency of participation.
These weights form a distribution over $[N]$:
\begin{align}
    \sum_{i=1}^{N} \mathbb{P}(i \in S^t)
    &= \sum_{i=1}^{N} \mathbb{E}\big[\mathds{1}\{i \in S^t\}\big]
    = \mathbb{E}\Bigg[\sum_{i=1}^{N} \mathds{1}\{i \in S^t\}\Bigg]\nonumber \\
    &= \mathbb{E}\big[|S^t|\big] = M.
\end{align}
We may then plausibly take
\begin{equation}\label{canonicalP}
    p_i := \frac{\mathbb{P}(i \in S^t)}{M} = \frac{1}{M}\sum_{S}\mathbb{P}(S^t=S)\,\mathds{1}\{i\in S\},\,\forall i \in [N].
\end{equation}
Since the selection dynamics are stationary, the weights $p_i$ do not depend on $t$. The distribution $\mathcal{P} := (p_1, \dots, p_N)$ is shaped by the selection budget $M$: for $M = N$ (full participation), $p_i = 1/N$ for all $i$ regardless of $\mathcal{R}$; for $M = 1$, $\mathcal{P}$ coincides with the RAM's singleton sampling distribution.

\vspace{4pt}
\noindent\textbf{General case: variable-size RAM.}
When $|S^t|$ is not fixed, the normalized weights remain well defined through a convex combination over subset sizes:
\begin{equation}\label{eq:general_p}
    \boxed{\,p_i = \sum_{S \subseteq [N]} \frac{\mathcal{R}(S)}{|S|}\, \mathds{1}\{i \in S\}\,,} \quad \text{for all } i \in [N],
\end{equation}
where $\mathcal{R}(S)$ is the probability of selecting the subset $S$ and, without loss of generality, $\mathcal{R}(\emptyset)=0$. Indeed, conditioning on the selected subset and using that a uniformly averaged subset of size $|S|$ assigns weight $1/|S|$ to each of its members,
\[
p_i = \sum_{S\subseteq[N]}
\underbrace{\tfrac{1}{|S|}\,\mathds{1}\{i \in S\}}_{\mathbb{P}(i \text{ weighted}\mid S^t=S)}\,
\underbrace{\mathcal{R}(S)}_{\mathbb{P}(S^t=S)},
\]
which defines a valid probability distribution over $[N]$ and accounts for the expected relative contribution of each client across all sampled subsets. Expression~\eqref{eq:general_p} contains the fixed-size formula~\eqref{canonicalP} as the special case $\mathcal{R}(\{S:|S|=M\})=1$, and permits unrestricted subset selection. The convergence guarantees of Section~\ref{Convergence} hold under~\eqref{eq:general_p} verbatim, since the participation model enters the analysis only through Lemma~\ref{lemma:s2m}, which is proved directly from~\eqref{eq:general_p}. For ease of comparison with state-of-the-art algorithms that assume fixed-size participation, we focus henceforth on the case in which exactly $M$ clients are selected per round.

\begin{figure}[!t]
    \centering
    \begin{tikzpicture}[node distance=1.2cm and 1.6cm, >=stealth, every node/.style={font=\footnotesize}]
    \node[draw, ellipse, align=center] (params) {Input\\$\theta_1^t,\dots,\theta_N^t$};
    \node[draw, rectangle, minimum width=2cm, minimum height=1cm, right=0.9cm of params, align=center] (RAM) {Stationary\\Erasure Channel};
    \node[draw, ellipse, right=0.9cm of RAM, align=center] (output) {Output\\$\theta_{i_1}^t, \dots, \theta_{i_M}^t$};
    \draw[->, thick] (params) -- (RAM);
    \draw[->, thick] (RAM) -- (output);
    \node[below=0.15cm of RAM] {Survival probability $\propto p_i = \tfrac{\mathbb{P}(i \in S^t)}{M}$};
    \end{tikzpicture}
    \caption{Canonical RAM structure: the $N$ inputs $\theta_1, \dots, \theta_N$ enter a stationary $(N,M)$-ary erasure channel (the RAM), producing $M$ surviving outputs (a random subset of the inputs), where the survival probability of each $\theta_i$ is proportional to $p_i$.}
    \label{fig:ram_erasure_channel}
    \vspace{-12pt}
\end{figure}
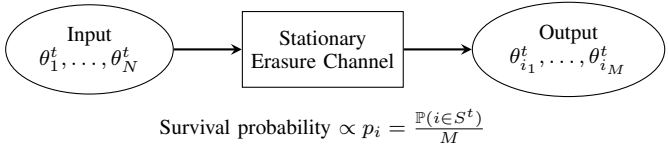

\vspace{-6pt}
\subsection{Distributionally Robust Federated Learning (DRFL)}\label{DRFL}

The ineffectiveness of agnostic FedAvg under client drift may be understood as a failure to generalize to a statistically balanced treatment of clients and their datasets. Even when FedAvg solves its own objective optimally, that objective may be biased toward particular clients or sub-datasets and far from the intended learning outcome. This is a \emph{distribution shift in the training objective}, caused by client drift arising from the RAM, from data inhomogeneity, or from both: clients' loss contributions are weighted unequally, compromising fairness and degrading performance for underrepresented clients. This is precisely the situation of~\eqref{EmpricialRiskNeutralFL} (equivalently~\eqref{BatchFL}) under the induced weighting~\eqref{canonicalP}.

Accordingly, we adopt a \emph{distributionally robust federated learning} (DRFL) formulation. Instead of optimizing against the fixed distribution $\mathcal{Q}^b$, we minimize the worst-case expected loss over a distributional neighborhood of $\mathcal{Q}^b$:
\begin{equation}\label{DRO}
\boxed{\inf_\theta\sup_{\mathcal{Q}\in\mathds{U}_\epsilon(\mathcal{Q}^b)}\mathbb{E}_{\xi\sim \mathcal{Q}}[f(\theta; \xi)],}
\end{equation}
where the \emph{ambiguity set} $\mathds{U}_\epsilon(\mathcal{Q}^b) = \{\mathcal{Q} : \mathrm{d}(\mathcal{Q}, \mathcal{Q}^b)\leq\epsilon\}$ is a distributional ``ball" of radius $\epsilon$ centered at $\mathcal{Q}^b$ in a distributional distance or divergence $\mathrm{d}$, both of which are specified in Section~\ref{ProposedApproach}. Whereas~\eqref{BatchFL} weights individual data points by the availability-induced distribution, Problem~\eqref{DRO} optimizes uniformly over a set of nearby, and potentially more desirable, distributions.

By construction, every $\mathcal{Q}\in\mathds{U}_\epsilon(\mathcal{Q}^b)$ incurs an average loss no larger than the adversarial loss $\sup_{\mathcal{Q}\in\mathds{U}_\epsilon(\mathcal{Q}^b)}\mathbb{E}_{\xi\sim \mathcal{Q}}[f(\theta; \xi)]$; minimizing the latter therefore yields a parameter that performs well for all $\mathcal{Q}\in\mathds{U}_\epsilon(\mathcal{Q}^b)$ simultaneously. As $\epsilon$ grows, the set $\mathds{U}_\epsilon(\mathcal{Q}^b)$ enlarges, admitting more pessimistic distributions and, at the same time, increasing the likelihood that it contains distributions more desirable than $\mathcal{Q}^b$; a performance trade-off is thus implicit in the choice of $\epsilon$. In the limit $\epsilon \rightarrow \infty$ the adversarial loss approaches $\sup_\xi f(\theta ; \xi)$, whereas $\epsilon=0$ recovers the nominal loss $\mathbb{E}_{\xi\sim \mathcal{Q}^b}[f(\theta; \xi)]$. Problem~\eqref{DRO} therefore spans the range between classical and maximally averse FL, providing a principled means of trading average-case behavior against robustness to distribution shift induced by general client drift.


DRFL in the present setting poses an additional natural difficulty: $\mathcal{Q}^b$ is unknown and cannot be manipulated, being implied by random process modeled by RAM and modeled through the RAM. We resolve this by exploiting the convex (Fenchel) duality between distributionally robust functionals and coherent risk measures, developed next.

\vspace{-6pt}
\section{Proposed Approach: \textsc{FedeRage}}\label{ProposedApproach}

\begin{algorithm}[t]
\caption{Federated Risk-Averse Averaging (\textsc{FedeRage})}
\label{alg:federage}
\begin{algorithmic}[1]
\State \textbf{Initialize:} $\theta_i^{0} = \mathbf{0}$, $\beta_i^{0} = 0$ for all $i \in [N]$; $T, H, \eta_\theta, \eta_\beta > 0$; 
\State \hspace{44.5pt}$\alpha \in (0,1]$, $\gamma \in [0,1]$; compact sets $\mathcal{C}, \mathcal{B}$.
\For{$t = 1, \ldots, TH$}
  \If{$t \bmod H = 0$} \Comment{\textbf{\textit{Global communication}}}
    \State Clients transmit local parameters through RAM.
    \State Server samples active set $S^{t} \sim \mathcal{R}$.
    \State Server \textbf{aggregates}: $\displaystyle{\begin{bmatrix} \hat{\theta}^{t} \\ \hat{\beta}^{t} \end{bmatrix} = \frac{1}{|S^t|}\sum_{i \in S^{t}} \begin{bmatrix} \theta_i^{t-1} \\ \beta_i^{t-1} \end{bmatrix}}$; 
    \State Server \textbf{broadcasts} to all clients.
  \Else \Comment{\textbf{\textit{Local updates}}}
    \State Each client draws $\xi_i^t \sim \mathcal{Q}_i^b$;
    \vspace{2pt}
    \State Computes $w_i^{t} = (1-\gamma) + \dfrac{\gamma}{\alpha}\,\mathds{1}\{f(\theta_i^{t-1}; \xi_i^t) \ge \beta_i^{t-1}\}$;
    \vspace{-7pt}
    \State Performs stochastic subgradient step:
    \vspace{2pt}
    \State $\begin{bmatrix} \theta_i^{t} \\ \beta_i^{t} \end{bmatrix} = \Pi_{\mathcal{C} \times \mathcal{B}}\left\{\begin{bmatrix} \theta_i^{t-1} \\ \beta_i^{t-1} \end{bmatrix} - \begin{bmatrix} \eta_\theta w_i^{t} \nabla_\theta f_i(\theta_i^{t-1}; \xi_i^t) \\ \eta_\beta(1 - w_i^{t}) \end{bmatrix}\right\}$.
  \EndIf
\EndFor
\State \textbf{Return:} $\frac{1}{T} \sum_{\tau=1}^{T} \hat{\theta}^{\tau H}, \frac{1}{T} \sum_{\tau=1}^{T} \hat{\beta}^{\tau H}$.
\end{algorithmic}
\end{algorithm}

\subsection{DRO via the Conditional Value-at-Risk}

Consider a random element $\xi$ with base distribution $\mathcal{Q}'$, and let $\xi \mapsto f(\xi) \in \mathbb{R}$ be integrable with respect to $\mathcal{Q}'$, so that $f(\xi)$ may be regarded as a random cost. The \emph{Conditional Value-at-Risk} (CVaR) of $f(\xi)$ at level $\alpha$ may be defined as~\cite{RockafellarUryasev2000,shapiro2009lectures}
\begin{equation}\label{CVARDef}
\cvar^\alpha_{\xi\sim \mathcal{Q}'}[f(\xi)] \triangleq \inf_{\beta \in \mathbb{R}} \left[\beta + \frac{1}{\alpha}\, \mathbb{E}_{\xi \sim \mathcal{Q}'}\big[(f(\xi) - \beta)_+\big]\right],
\end{equation}
where $(\cdot)_+=\max\{\cdot, 0\}$ and $\alpha \in (0,1]$ is the \emph{risk or confidence level}. Equivalently, CVaR is the average of the loss over its worst (upper) $\alpha$-tail reading
\begin{equation}\label{RiskAverseCVaR}
\cvar^\alpha[f(\xi)] = \frac{1}{\alpha}\int_{1-\alpha}^{1}\mathrm{VaR}_{u}\big[f(\xi)\big]\,\mathrm{d}u ,
\end{equation}
where $\mathrm{VaR}_u$ denotes the $u$-quantile. For a \emph{continuous} distribution this reduces to a conditional-tail-expectation form
\[
\cvar^\alpha[f(\xi)] = \mathbb{E}\big[f(\xi) \mid f(\xi) \geq \beta^*_\alpha\big], \quad \mathbb{P}\big(f(\xi) \geq \beta^*_\alpha\big) = \alpha,
\]
with $\beta^*_\alpha$ an optimal $\beta$ in~\eqref{CVARDef}. For general (in particular, discrete) distributions this identity may fail on account of an atom at the quantile, whereas~\eqref{RiskAverseCVaR} always holds.

CVaR is a coherent risk measure---convex, monotone, translation equivariant, and positively homogeneous~\cite{shapiro2009lectures}---and admits the dual representation widely used in DRO \begin{equation}\label{DualCVARRepresentation}
\cvar^\alpha[f(\xi)] = \sup_{\mathcal{Q} \in \mathds{U}_{\alpha}} \mathbb{E}_{\mathcal{Q}}[f(\xi)],
\end{equation}
where the ambiguity set $\mathds{U}_{\alpha}$ is
\begin{align*}
\mathds{U}_{\alpha} &= \left\{ \mathcal{Q} \ \middle|\ \frac{\mathrm{d}\mathcal{Q}}{\mathrm{d}\mathcal{Q}'} \in \left[0, \tfrac{1}{\alpha}\right] \text{ a.e.-}\mathcal{Q}' \right\},
\end{align*}
and $\mathcal{Q}'$ is the nominal (reference) distribution. CVaR is thus a worst-case expectation over a likelihood-ratio (R\'enyi-type) divergence ball of radius $\log(1/\alpha)$, and instantiates~\eqref{DRO} with $\epsilon=\log(1/\alpha)$. Varying $\alpha$ interpolates continuously between the expectation and the essential supremum:
\[
\cvar^{1}[f(\xi)] = \mathbb{E}[f(\xi)], \quad \lim_{\alpha \rightarrow 0^+}\cvar^\alpha[f(\xi)] = \esssup f(\xi).
\]
As $\alpha\to 0^+$, the likelihood-ratio bound $1/\alpha$ diverges and the ambiguity set expands toward the worst case, whereas $\alpha=1$ forces $\mathrm{d}\mathcal{Q}/\mathrm{d}\mathcal{Q}'=1$ a.e.\ and recovers the nominal expectation. This renders CVaR a natural risk-sensitive functional for robust optimization.

\vspace{-6pt}
\subsection{DRFL via the CVaR: \textsc{FedeRage}}

To instantiate~\eqref{DRO} in a manner computable without access to $\mathcal{Q}^b$, we combine the CVaR under the RAM-induced distribution with the risk-neutral objective, replacing~\eqref{BatchFL} by the mean--CVaR problem
\begin{equation}\label{RiskAverseFL}
    \boxed{\inf_{\theta \in \Theta}\Big[ (1-\gamma)\, \mathbb{E}_{\xi \sim \mathcal{Q}^b}[f(\theta; \xi)] + \gamma\, \cvar^\alpha_{\xi \sim \mathcal{Q}^b}[f(\theta; \xi)]\Big],}
\end{equation}
with $\alpha\in(0,1]$ and $\gamma \in [0,1]$ tunable hyperparameters. Using the variational form~\eqref{CVARDef}, Problem~\eqref{RiskAverseFL} is equivalent to the joint problem
\begin{align*}
    \inf_{\theta \in \Theta,\, \beta \in \mathbb{R}}\ \fag(\theta,\beta),
\end{align*}
where
\begin{align}\label{eq:fag}\hspace{-6bp}\fag(\theta,\beta) \hspace{-2bp}:=\hspace{-1bp}\mathbb{E}\bigg[(1 \hspace{-1bp}-\hspace{-1bp}\gamma)f(\theta; \xi) \hspace{-1bp}+\hspace{-1bp} \gamma \Big[\beta \hspace{-1bp}+\hspace{-1bp} \tfrac{1}{\alpha}\big(f(\theta; \xi) \hspace{-1bp}-\hspace{-1bp} \beta\big)_+\Big]\bigg].\hspace{-6bp}
\end{align}
The expectation over $\mathcal{Q}^b$ is internal to the definition of $\fag$; consequently, $\fag(\theta,\beta)$ denotes a deterministic quantity and no outer expectation is written anywhere below.

Relative to agnostic FedAvg (Algorithm~\ref{alg:fedavg}), the only modification required by~\eqref{eq:fag} is the introduction of a single scalar $\beta \in \mathbb{R}$, maintained locally alongside the model parameters, so that Algorithm~\ref{alg:federage} retains the two-phase structure of its risk-neutral counterpart. Both local updates are driven by the same scalar risk weight $w_i^t=(1-\gamma)+\frac{\gamma}{\alpha}\mathds{1}\{f_i\ge\beta\}$ of Line~8: the update direction in $\theta$ is $w_i^t\nabla_\theta f_i$, whereas the subgradient with respect to $\beta$ is selected as $\gamma\big[1-\frac{1}{\alpha}\mathds{1}\{f_i\ge\beta\}\big]=1-w_i^t$, and in particular involves no derivative of $f_i$ in $\beta$. We project $\beta$ onto a compact interval $\mathcal{B}$, which is done without loss of generality, since the optimizer $\beta^\star$ of~\eqref{CVARDef} is a finite value-at-risk lying in the range of the loss. Each local step uses a single model mini-batch $\xi_i^t\sim\mathcal{Q}_i^b$, i.e., \textit{algorithmic} mini-batch size one, which permits tight control of the upper tail of the loss.

\vspace{4pt}
\noindent\textit{Comparison relative to SCAFFOLD~\cite{SCAFFOLD}.}
Several state-of-the-art methods rely on gradient-aggregation mechanisms that require communicating additional model-sized vectors, widening the communication channel and enlarging the attack surface. SCAFFOLD~\cite{SCAFFOLD} is the canonical instance: each client maintains and transmits a control variate of the same dimension $d'$ as the model, which is an estimate of that client's local gradient. This approximately doubles the per-round uplink and, more significantly, exposes a second, gradient-valued quantity per client---precisely the object exploited by gradient-inversion and membership-inference attacks~\cite{Hitaj2017,Geyer2017}, and one that is not protected by the averaging applied to the model update. \textsc{FedeRage} transmits instead a single scalar $\beta_i$, an estimate of a quantile of the client's own loss distribution: it is smaller by a factor $d'$, carries no directional information about the data, and has sensitivity bounded by the diameter of $\mathcal{B}$, so that differential-privacy noise calibrated to it costs a small fraction of the utility required to privatize a $d'$-dimensional control variate. Both $\theta_i$ and $\beta_i$ reach the server only through plain averages, so \textsc{FedeRage} is compatible with secure aggregation without modification. Robustness to heterogeneity is thus obtained through the \emph{local objective} rather than through additional shared state.

\begin{figure*}[t]
    \centering
    \begin{minipage}[b]{0.49\textwidth}
        \centering
        \includegraphics[width=\textwidth]{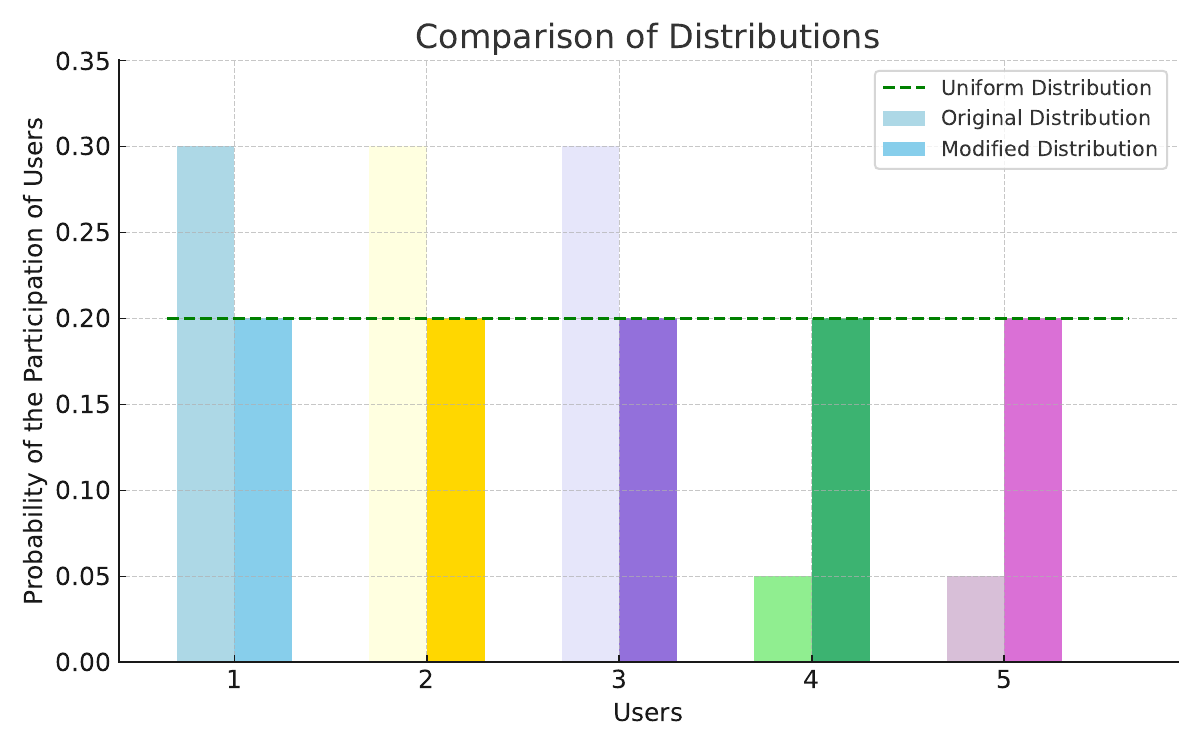}
        \caption*{(a) Client availability probabilities (light colors), the uniform distribution (dashed green line), and the effective availability distribution induced by CVaR with $\gamma=0.25$ and $\alpha = 0.01$. The change in the effective distribution is implicit, arising from the dual representation of CVaR.}
    \end{minipage}
    \hfill
    \begin{minipage}[b]{0.49\textwidth}
        \centering
        \includegraphics[width=\textwidth]{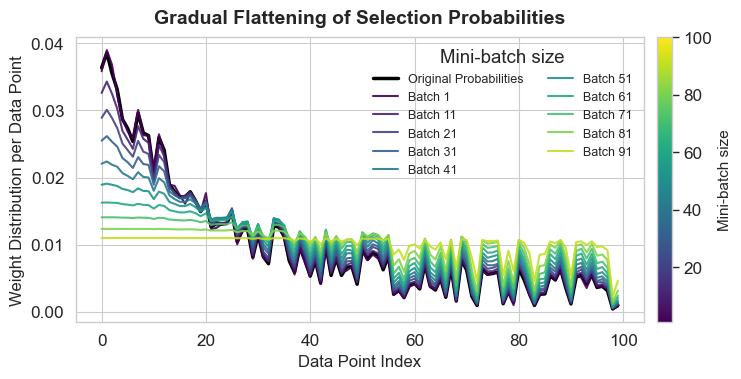}
        \caption*{(b) Given $100$ data points with an initial weight distribution (black), the effective weights change with the mini-batch size. As the mini-batch approaches the full batch, the effective weights become uniform (green). This loss of variability suggests that too large a mini-batch degrades performance by reducing the effective degrees of freedom.}
    \end{minipage}
    \caption{(a) Effect of CVaR on the effective client-availability distribution. (b) Effect of mini-batch size on effective selection probabilities. Both illustrate how distributional variability influences federated optimization dynamics.}
    \label{fig:cvar_vs_batch_effects}
    \vspace{-12pt}
\end{figure*}

\vspace{-8pt}
\subsection{Risk-Averse (Primal) Interpretation of \textsc{FedeRage}}

To interpret the objective~\eqref{eq:fag}, consider again the example of
Fig.~\ref{fig:motivation}, with two infrequent clients ($4$ and $5$) alongside
three frequently participating ones ($1$--$3$). Agnostic FedAvg minimizes the
participation-weighted loss $f(\theta)=\sum_{i=1}^{N}p_i f_i(\theta)$
of~\eqref{EmpricialRiskNeutralFL}, in which the induced weights $p_4,p_5$ are
small, so that the local losses $f_4(\theta)$ and $f_5(\theta)$ contribute little
to the global objective even when they remain large. Minimization of $f$ is
therefore driven by clients $1$--$3$, and the resulting model fits the dominant
clients while neglecting the rare patterns held by the underrepresented ones.

The mean--CVaR integrand of~\eqref{eq:fag}
\begin{align}
\nonumber
g_{\alpha,\gamma}(\theta, \beta; \xi) &:= (1 - \gamma)\, f(\theta; \xi)
+ \gamma \Big( \beta + \tfrac{1}{\alpha}\big(f(\theta; \xi) - \beta\big)_+ \Big),
\end{align}
whose expectation over $\xi\sim\mathcal{Q}^b$ is $\fag(\theta,\beta)$, modifies
the influence of each mini-batch $\xi$ in a risk-sensitive manner. Recall that
$\mathcal{Q}^b=\sum_{i=1}^N p_i\,\mathcal{Q}^b_i$, so that a mini-batch of
client $i$ carries nominal mass $p_i/N_i^b$ under $\mathcal{Q}^b$; the threshold
$\beta$ is accordingly a quantile (estimate) of the loss under the \emph{global} mixture,
maintained locally as $\beta_i$ between communication rounds and averaged at the
server (Algorithm~\ref{alg:federage}). Whenever the mini-batch loss
$f(\theta; \xi)$ exceeds $\beta$, which at optimality is an $(1-\alpha)$-quantile under $\mathcal{Q}^b$, the corresponding subgradient is upweighted as
\begin{align}\nonumber
\nabla_\theta\, g_{\alpha,\gamma}(\theta, \beta; \xi)
= \underbrace{\big(1 - \gamma + \tfrac{\gamma}{\alpha}\,
\mathds{1}\{f(\theta; \xi) \geq \beta\}\big)}_{=:\,w(\theta,\beta;\xi)}\,
\nabla_\theta f(\theta; \xi),
\end{align}
where $\nabla_\theta f(\theta;\xi)$ denotes the measurable subgradient selection
fixed in Section~\ref{Convergence} and $w(\theta,\beta;\xi)$ is precisely the
scalar risk weight $w_i^t$ of Line~8 of Algorithm~\ref{alg:federage}, evaluated
at $(\theta_i^{t-1},\beta_i^{t-1};\xi_i^t)$.

The expression $\omega(\theta,\beta;\xi)$ constitutes a risk-weighted
prioritization mechanism. If the mini-batches of client $i$ consistently incur
large losses---whether because of infrequent RAM selection (small $p_i$) or
because of misalignment with the global model (client drift)---they are more
likely to exceed the threshold $\beta$, and their subgradients are scaled by a
factor of up to $\kappa(\alpha,\gamma)=(1-\gamma)+\gamma/\alpha$, amplifying
their contribution to the update. Underrepresented or difficult data thereby
receive proportionally greater attention, without the algorithm ever using, or
even forming, an estimate of $\{p_i\}$.

The mean--CVaR objective is furthermore a smooth surrogate for the worst-case
loss. In the limit $\gamma = 1$, $\alpha \to 0^{+}$, the CVaR term tends to
$\esssup_{\xi\sim\mathcal{Q}^b} f(\theta;\xi)$, so that
minimizing~\eqref{eq:fag} approaches minimization of the largest mini-batch loss
in the support of $\mathcal{Q}^b$, up to smoothing; since that support is the
union of the clients' mini-batch collections $\{D_i^b\}_{i\in[N]}$, this steers
the global model toward regions in which the largest client losses
$\max_{i\in[N]} f_i(\theta)$ are reduced, thereby promoting an equity across
clients that the risk-neutral formulation~\eqref{EmpricialRiskNeutralFL} does
not enforce.

\vspace{-4pt}
\subsection{Distributionally Robust (Dual) Interpretation of \textsc{FedeRage}}
\label{sec:dual_federage}

We now use the dual representation~\eqref{DualCVARRepresentation} to interpret the CVaR term in~\eqref{eq:fag} in a complementary distributionally robust manner. Recall the expected loss over RAM-distributed mini-batches
\begin{align}
\mathbb{E}_{\xi \sim \mathcal{Q}^{b}} [f(\theta; \xi)]
&= \sum_{i=1}^N p_i \Big( \tfrac{1}{N_{i}^{b}} \sum_{j=1}^{N_{i}^{b}} f_i(\theta; \xi_i^j) \Big)
= \sum_{i=1}^N \sum_{j=1}^{N_{i}^{b}} q_i^j\, f_i(\theta; \xi_i^j), \nonumber
\end{align}
\vspace{-14pt}

\noindent where the weights $q_i^j = p_i/N_i^b$ constitute $\mathcal{Q}^b$. By the dual form of CVaR, it follows that
\begin{align}
\cvar_{\xi \sim \mathcal{Q}^b}^\alpha[f(\theta; \xi)]
&= \sup_{\hat{\mathcal{Q}} \in \mathds{U}_{\alpha}(\mathcal{Q}^b)} \mathbb{E}_{\xi \sim \hat{\mathcal{Q}}}[f(\theta; \xi)] \nonumber\\
&= \sup_{\hat{\mathcal{Q}} \in \mathds{U}_{\alpha}(\mathcal{Q}^b)} \sum_{i=1}^N \sum_{j=1}^{N_{i}^{b}} \hat{q}_i^j\, f_i(\theta; \xi_i^j), \nonumber
\end{align}
where the mini-batch ambiguity set is
\begin{align}
\mathds{U}_{\alpha}(\mathcal{Q}^b) = \Big\{ \hat{\mathcal{Q}} \,\Big|\, \tfrac{\hat{q}_i^j}{p_i} \leq \tfrac{1}{N_{i}^{b} \alpha},\ i \in [N],\, j \in [N_i^b] \Big\}. \nonumber
\end{align}
\vspace{-7pt}

Expanding the loss over individual data points, we have
\begin{align}
&\mathbb{E}_{\xi\sim\hat{\mathcal{Q}}}\big[f(\theta;\xi)\big] \nonumber
\\&=\sum_{i=1}^{N}\sum_{j=1}^{N_{i}^{b}}\hat{q}_{i}^{j}\sum_{(X,Y)\in\xi_{i}^{j}}\tfrac{1}{b}\,\ell(m(X,\theta),Y) \nonumber
\\&=\sum_{i=1}^{N}\sum_{k=1}^{n_{i}}
\underbrace{\Bigg[\tfrac{1}{b}\sum_{j=1}^{N_{i}^{b}}\hat{q}_{i}^{j}\,\mathds{1}\{(X_{i}^{k},Y_{i}^{k})\in\xi_{i}^{j}\}\Bigg]}_{\triangleq\,\tilde{q}_{i}^{k}}\ell(m(X_{i}^{k},\theta),Y_{i}^{k}) \nonumber
\\&=\sum_{i=1}^{N}\sum_{k=1}^{n_{i}}\tilde{q}_{i}^{k}\,\ell(m(X_{i}^{k},\theta),Y_{i}^{k}), \nonumber
\end{align}
\vspace{-7pt}

\noindent where the \emph{inherited} weights $\{\tilde{q}_i^k\}$ induce a reweighted importance distribution $\tilde{\mathcal{Q}}$ over individual data points with CVaR
\vspace{-7pt}

\begin{align}
&\cvar_{\xi\sim \mathcal{Q}^{b}}^{\alpha}\big(f(\theta;\xi)\big) \nonumber \\
&\qquad=
\sup_{\tilde{\mathcal{Q}}\in\tilde{\mathds{U}}_{\alpha}(\mathcal{Q}^{b})}\ \mathbb{E}_{(I,(X,Y))\sim\tilde{\mathcal{Q}}}\big[\ell(m(X,\theta),Y)\big],\nonumber
\end{align}
\vspace{-7pt}

with the inheritance uncertainty set
\begin{align}
\tilde{\mathds{U}}_{\alpha}(\mathcal{Q}^{b})=\Big\{\tilde{\mathcal{Q}}\ \Big|\
&\tilde{q}_{i}^{k}=\tfrac{1}{b}\textstyle\sum_{j=1}^{N_{i}^{b}}\hat{q}_{i}^{j}\,\mathds{1}\{(X_{i}^{k},Y_{i}^{k})\in\xi_{i}^{j}\},\nonumber\\
&\tfrac{\tilde{q}_{i}^{k}}{p_{i}}\le\tfrac{1}{n_{i}\alpha},\ i\in[N],\,k\in[n_i],\nonumber\\
&\{\hat{q}_{i}^{j}\}=\hat{\mathcal{Q}}\in\mathds{U}_{\alpha}(\mathcal{Q}^{b})\Big\}.\nonumber
\end{align}
In words, the distributional robustness that CVaR confers on the mini-batch distribution $\hat{\mathcal{Q}}$ is \emph{inherited} by the induced point-level distribution $\tilde{\mathcal{Q}}$: maximizing over the mini-batch ambiguity set $\mathds{U}_{\alpha}(\mathcal{Q}^b)$ is equivalent to maximizing over the structured point-level set $\tilde{\mathds{U}}_{\alpha}(\mathcal{Q}^{b})$, which couples the CVaR reweighting with the RAM-induced participation weights $p_i$. \textsc{FedeRage} is therefore distributionally robust at the level of individual data points, hedging against the client drift produced jointly by data heterogeneity and skewed availability. This reweighting is realized \emph{implicitly}, through the local objective, and never appears as an explicit aggregation weight at the server, which is what allows the algorithm to operate without any knowledge of the participation law.

In the special case $b = 1$, in which mini-batches are single points, the two levels coincide:
\begin{align}
\cvar_{\xi \sim \mathcal{Q}^b}^\alpha\big(f(\theta; \xi)\big)
= \cvar_{(I, (X, Y)) \sim \tilde{\mathcal{Q}}}^\alpha\big(\ell(m(X, \theta), Y)\big), \nonumber
\end{align}
whereas for $b > 1$ the right-hand side is a relaxation induced by overlapping mini-batch structures (Fig.~\ref{fig:cvar_vs_batch_effects}). Incorporating CVaR thus reweights samples according to their contribution to the upper quantile of the loss distribution. At each communication round, harder examples---data points or clients with higher loss---are prioritized, so that the algorithm attends to poorly performing or underrepresented clients, improving robustness to data heterogeneity and worst-case performance.

\section{Convergence for Convex Losses}\label{Convergence}

This section develops the convergence analysis in two stages, with every proof given in place. We first treat the risk-neutral case $\gamma=0$, i.e., agnostic FedAvg applied to the induced objective~\eqref{BatchFL} (Sections~\ref{Convergence}-A through~\ref{Convergence}-C), and then extend the argument to the mean--CVaR objective $\fag$ for $\gamma>0$ and $\alpha\in(0,1]$ (Section~\ref{Convergence}-D). The separation is deliberate: the risk-neutral analysis isolates the effect of unknown, variable-size participation, and the risk-averse extension is then shown to cost exactly one explicit constant.

Two sources of randomness enter. First, for each client $i\in[N]$, the random element $\xi_i$ denotes a mini-batch of size $b$ drawn uniformly without replacement from client $i$'s dataset; at time $t$ we write $\xi_i^t$, and mini-batches are drawn independently across clients and rounds. Second, $S^t \subseteq [N]$ denotes the subset of clients selected at round $t$ under the RAM, i.i.d.\ across rounds and independent of $\{\xi_i^t\}$. Both sources are carried through every round, with the understanding that the sampled $\xi_i^t$ are unused during global rounds and the outcome $S^t$ is unused during local rounds. We define the natural filtration $\{\mathcal{F}_t\}_t$,
\[
\mathcal{F}_t := \sigma\big( \theta_i^s, \xi_i^s, S^s : s \leq t,\, i \in [N] \big),
\]
which captures the model states, mini-batch draws, and participation history up to round $t$; the iterates $\theta_i^t$ are adapted to $\{\mathcal{F}_t\}$ and depend on $\mathcal{F}_{t-1}$ only through $\theta_i^{t-1}$ and the fresh randomness $\xi_i^t$ or $S^t$.

Since the losses are convex and possibly nonsmooth, $\nabla f_i(\theta)$ and $\nabla f(\theta;\xi)$ denote throughout a fixed measurable selection from the subdifferentials $\partial f_i(\theta)$ and $\partial_\theta f(\theta;\xi)$, chosen so that the unbiasedness required in Assumption~\ref{assump:gradient_bound} holds; all inequalities below use only the convexity (subgradient) inequality and therefore remain valid for any such selection. Our structural hypotheses are the following two; the bounded-variance condition used in the analysis is derived from them rather than assumed.

\vspace{-3pt}
\begin{assumption}[\textbf{Convexity}]
\label{assump:convex}
The ``instantaneous" losses   $f(\cdot;\xi)$ are convex on $\Theta$ for every mini-batch $\xi$.
\end{assumption}
\vspace{-9pt}
\begin{assumption}[\textbf{Unbiasedness and Bounded Second Moment}]
\label{assump:gradient_bound}
For each $i \in [N]$ and every $\theta\in\Theta$, the stochastic subgradient is conditionally unbiased, $\mathbb{E}_{\xi_i}[\nabla f(\theta;\xi_i)]=\nabla f_i(\theta)$, and
\[
\sup_{\theta \in \Theta}\mathbb{E}_{\xi_i} \big[ \| \nabla f(\theta; \xi_i) \|^2 \big] \leq G^2 .
\]
\end{assumption}

\begin{lemma}[\textbf{Bounded Local Gradient Variance}]
\label{lemma:variance}
Under Assumption~\ref{assump:gradient_bound}, for every $i \in [N]$,
\[
\sup_{\theta \in \Theta}\mathbb{E}_{\xi_i} \big[ \| \nabla f(\theta; \xi_i) - \nabla f_i(\theta) \|^2 \big] \leq \sigma_i^2 := G^2 - \inf_{\theta\in\Theta}\|\nabla f_i(\theta)\|^2,
\]
and therefore the aggregate variance bound $\sigma^2 := \sum_{i=1}^N p_i \sigma_i^2$ satisfies $\sigma^2\le G^2$.
\end{lemma}
\vspace{-8pt}
\begin{proof}
Fix $\theta\in\Theta$ and set $X:=\nabla f(\theta;\xi_i)$, so that $\mathbb{E}X=\nabla f_i(\theta)$ by Assumption~\ref{assump:gradient_bound}. The bias--variance identity yields
\[
\mathbb{E}\|X-\mathbb{E}X\|^2=\mathbb{E}\|X\|^2-\|\mathbb{E}X\|^2\le G^2-\|\nabla f_i(\theta)\|^2 .
\]
Taking the supremum over $\theta\in\Theta$ gives the first claim, and the bound on $\sigma^2$ follows because $\{p_i\}_{i\in[N]}$ is a probability distribution and each $\sigma_i^2\le G^2$.
\end{proof}

We retain the symbol $\sigma^2$ below, as it is the sharper of the two constants and isolates the contribution of sampling noise; every statement below remains valid with $\sigma^2$ replaced by $G^2$.

Throughout, $\mathcal{C}\subseteq\Theta$ is convex and compact with Euclidean projection $\Pi_\mathcal{C}(\cdot)$, and $\theta^*$ denotes a minimizer of $f$ over $\mathcal{C}$, which exists by continuity of $f$ and compactness of $\mathcal{C}$. Since each $f_i$ is convex and finite on a neighborhood of the compact set $\mathcal{C}$, it is Lipschitz on $\mathcal{C}$; we let $\ell$ denote a common Lipschitz constant.

\vspace{-6pt}
\subsection{Preliminary Lemmata}

We first analyze a single projected stochastic subgradient step at the client side. Throughout, $g_i^t := \nabla f(\theta_i^{t-1};\xi_i^t)$, so that $\mathbb{E}[g_i^t\mid\mathcal{F}_{t-1}]=\nabla f_i(\theta_i^{t-1})$ by Assumption~\ref{assump:gradient_bound}; this is the update direction of Line~9 of Algorithm~\ref{alg:fedavg}. Conditional expectations are unrolled by the tower property.

\begin{lemma}[\textbf{One-Step Progress}]
\label{lemma:onestep}
Let $\theta_i^{t} = \Pi_{\mathcal{C}}(\theta_i^{t-1} - \eta g_i^{t})$ for $i\in[N]$ and $\eta>0$. Under Assumptions~\ref{assump:convex} and~\ref{assump:gradient_bound} and Lemma~\ref{lemma:variance}, at every local round,
\begin{align}
\mathbb{E}\big[ \| \theta_i^t - \theta^* \|^2 \mid \mathcal{F}_{t-1} \big]
&\leq \| \theta_i^{t-1} - \theta^* \|^2 \nonumber \\
&\quad - 2\eta \big(f_i(\theta_i^{t-1}) - f_i(\theta^*) \big) \nonumber \\
&\quad + 2\eta^2 (\sigma^2 + G^2). \label{eq:l1}
\end{align}
\end{lemma}
\vspace{-8pt}
\begin{proof}
Since $\theta^*\in\mathcal{C}$ and the projection onto a closed convex set is non-expansive,
\begin{align}
\|\theta_i^t - \theta^*\|^2
&= \| \Pi_{\mathcal{C}}(\theta_i^{t-1} - \eta g_i^t) - \Pi_{\mathcal{C}}(\theta^*) \|^2 \nonumber \\
&\leq \| \theta_i^{t-1} - \eta g_i^t - \theta^* \|^2 \nonumber \\
&= \| \theta_i^{t-1} - \theta^* \|^2 - 2\eta \langle g_i^t, \theta_i^{t-1} - \theta^* \rangle + \eta^2 \|g_i^t\|^2. \nonumber
\end{align}
Take $\mathbb{E}[\,\cdot\mid\mathcal{F}_{t-1}]$ on both sides. Writing $g_i^t = \nabla f_i(\theta_i^{t-1}) + \zeta_i^t$ with $\mathbb{E}[\zeta_i^t\mid\mathcal{F}_{t-1}]=0$, and noting that $\theta_i^{t-1}$ is $\mathcal{F}_{t-1}$-measurable,
\[
\mathbb{E}\big[ \langle g_i^t, \theta_i^{t-1} - \theta^* \rangle \mid \mathcal{F}_{t-1} \big]
= \langle \nabla f_i(\theta_i^{t-1}), \theta_i^{t-1} - \theta^* \rangle ,
\]
while $\|a+b\|^2\le 2\|a\|^2+2\|b\|^2$ together with Assumption~\ref{assump:gradient_bound} and Lemma~\ref{lemma:variance} gives
\[
\mathbb{E}\big[ \|g_i^t\|^2 \mid \mathcal{F}_{t-1} \big] \leq 2 \|\nabla f_i(\theta_i^{t-1})\|^2+2\,\sigma^2 \le 2 G^2 + 2 \sigma^2 .
\]
Combining the two displays,
\begin{multline}
\mathbb{E}\big[ \|\theta_i^t - \theta^*\|^2 \mid \mathcal{F}_{t-1} \big]
\leq \|\theta_i^{t-1} - \theta^*\|^2 \\
- 2\eta \langle \nabla f_i(\theta_i^{t-1}), \theta_i^{t-1} - \theta^* \rangle + 2\eta^2 (G^2 + \sigma^2). \nonumber
\end{multline}
The subgradient inequality for the convex function $f_i$ (Assumption~\ref{assump:convex}) gives $f_i(\theta_i^{t-1}) - f_i(\theta^*) \leq \langle \nabla f_i(\theta_i^{t-1}), \theta_i^{t-1} - \theta^* \rangle$, and substituting it yields~\eqref{eq:l1}.
\end{proof}

We next bound the divergence between two clients' parameter vectors within a communication round.

\begin{lemma}[\textbf{Local Parameter Divergence}]
\label{lemma:local_divergence}
Let $i, j \in [N]$ and let $\tau_i, \tau_j \in [SH, SH +H]$ be local steps between consecutive communication rounds $S$ and $S+1$. Under Assumption~\ref{assump:gradient_bound},
\[
\mathbb{E}\big[\|\theta_i^{\tau_i} - \theta_j^{\tau_j}\| \,\big|\, \mathcal{F}_{SH}\big] \leq 4\eta G H.
\]
\end{lemma}
\vspace{-8pt}
\begin{proof}
Assume without loss of generality that $\tau_i \geq \tau_j$, and set $\Delta_t := \theta_i^t - \theta_j^t$. At each local step, non-expansiveness of the projection gives
\begin{align}
\|\Delta_t\| &= \big\| \Pi_{\mathcal{C}}(\theta_i^{t-1}\! - \eta g_i^t) - \Pi_{\mathcal{C}}(\theta_j^{t-1}\! - \eta g_j^t) \big\| \nonumber\\
&\le \big\|\Delta_{t-1} - \eta (g_i^t - g_j^t)\big\| , \nonumber
\end{align}
so that $\|\Delta_t\|^2 \leq \|\Delta_{t-1}\|^2 - 2\eta \langle \Delta_{t-1}, g_i^t - g_j^t \rangle + \eta^2 \|g_i^t - g_j^t\|^2$. By Cauchy--Schwarz and Assumption~\ref{assump:gradient_bound},
\[
-2\eta \langle \Delta_{t-1}, g_i^t - g_j^t \rangle \leq 4\eta G \|\Delta_{t-1}\|, \quad \mathbb{E}\|g_i^t - g_j^t\|^2 \leq 4G^2 ,
\]
whence $\mathbb{E}\big[ \|\Delta_t\|^2 \mid \mathcal{F}_{t-1} \big] \leq \big( \|\Delta_{t-1}\| + 2\eta G \big)^2$, and, by Jensen's inequality, $\mathbb{E}[\|\Delta_t\|\mid\mathcal{F}_{t-1}] \le \|\Delta_{t-1}\| + 2\eta G$. Since $\theta_i^{SH} = \theta_j^{SH} = \hat{\theta}^{SH}$ implies $\Delta_{SH} = 0$, iterating this recursion and using the tower property yields $\mathbb{E}[\|\Delta_{SH+h}\|\mid\mathcal{F}_{SH}] \leq 2\eta G h$ for $h\in\{0,\dots,H\}$; in particular $\mathbb{E}[\|\theta_i^{\tau_j} - \theta_j^{\tau_j}\| \mid \mathcal{F}_{SH}] \leq 2\eta G (\tau_j - SH)$. Finally, by the triangle inequality and $\|\theta_i^{\tau_i}-\theta_i^{\tau_j}\|\le\sum_{s=\tau_j+1}^{\tau_i}\eta\|g_i^s\|$,
\begin{align}
\mathbb{E}\big[\|\theta_i^{\tau_i} - \theta_j^{\tau_j}\|\,\big|\,\mathcal{F}_{SH}\big]
&\leq \mathbb{E}\big[\|\theta_i^{\tau_j} - \theta_j^{\tau_j}\|\,\big|\,\mathcal{F}_{SH}\big] \nonumber\\
&\quad + \mathbb{E}\big[\|\theta_i^{\tau_i} - \theta_i^{\tau_j}\|\,\big|\,\mathcal{F}_{SH}\big] \nonumber \\
&\leq 2\eta G (\tau_j - SH) + \eta G (\tau_i - \tau_j) \nonumber\\
&\leq 3\eta G H . \nonumber
\end{align}
Symmetrizing in $i\leftrightarrow j$ and bounding by the worst case gives the stated constant $4\eta G H$.
\end{proof}

\begin{table*}[t!]
\centering
\begin{tabular}{c|cc|cc|cc}
\toprule
\textbf{Algorithm} & \multicolumn{2}{c|}{\textbf{MNIST}} & \multicolumn{2}{c|}{\textbf{FashionMNIST}} & \multicolumn{2}{c}{\textbf{CIFAR-10}} \\
                   & \textbf{Uniform} & \textbf{Non-Uniform} & \textbf{Uniform} & \textbf{Non-Uniform} & \textbf{Uniform} & \textbf{Non-Uniform} \\
\midrule
FedAvg    & 79.54 $\pm$ 8.70  & 78.19 $\pm$ 11.12  & 75.08 $\pm$ 5.17  & 61.82 $\pm$ 10.39  & 51.43 $\pm$ 2.19  & 50.61 $\pm$ 1.48  \\
\textsc{FedProx}   & 79.19 $\pm$ 9.03  & 77.85 $\pm$ 10.96  & 75.11 $\pm$ 5.22  & 61.82 $\pm$ 10.38  & 52.18 $\pm$ 0.74  & 50.25 $\pm$ 1.75  \\
Scaffold  & \bestmean{\bestband{96.95}}\textsuperscript{\textdagger} $\pm$ \bestband{0.26}  & 95.01 $\pm$ 0.50  & \bestmean{84.24} $\pm$ 2.36  & 77.36 $\pm$ 1.09  & 51.21 $\pm$ 2.71  & 50.57 $\pm$ 1.49  \\
\textsc{FedeRage}  & 96.02 $\pm$ 0.30  & \bestmean{\bestband{95.91}}\textsuperscript{\textdagger} $\pm$ \bestband{0.41}  & \bestband{83.99}\textsuperscript{\textdagger} $\pm$ \bestband{1.71}  & \bestmean{\bestband{83.95}}\textsuperscript{\textdagger} $\pm$ \bestband{1.03}  & \bestmean{\bestband{57.41}}\textsuperscript{\textdagger} $\pm$ \bestband{1.51}  & \bestmean{\bestband{56.34}}\textsuperscript{\textdagger} $\pm$ \bestband{0.61}  \\
\bottomrule
\end{tabular}
\caption{Test accuracy (\%) $\pm$ standard deviation for FedAvg, \textsc{FedProx}, Scaffold, and \textsc{FedeRage} on MNIST, FashionMNIST, and CIFAR-10 after $10{,}000$ communication rounds, under uniform and non-uniform client availability. Entries marked \textsuperscript{\textdagger} are statistically significant (two-sample $Z$-test at the $5\%$ level, computed on the final ten communication rounds). \textbf{Bold} indicates the best mean in each column; \textcolor{blue}{blue} indicates the best mean and standard deviation \emph{jointly}, i.e., the largest value of mean $-$ standard deviation, which is the relevant figure of merit when run-to-run stability is at stake.}
\label{tab:results}
\vspace{-10pt}
\end{table*}
Lipschitz continuity of the losses then yields the following.

\begin{corollary}[\textbf{Local Value Divergence}]
\label{corollary:lipschitz_difference}
For $i, j \in [N]$ and $\tau_i, \tau_j \in [SH, SH+H]$,
\[
\mathbb{E}\big[ |f_i(\theta_i^{\tau_i}) - f_i(\theta_j^{\tau_j})| \,\big|\, \mathcal{F}_{SH}\big] \leq 4\ell \eta G H.
\]
\end{corollary}
\vspace{-8pt}
\begin{proof}
By $\ell$-Lipschitz continuity of $f_i$ on $\mathcal{C}$ we have $|f_i(\theta) - f_i(\theta')| \leq \ell \|\theta - \theta'\|$ for all $\theta,\theta'\in\mathcal{C}$. Applying this to $\theta_i^{\tau_i},\theta_j^{\tau_j}\in\mathcal{C}$, taking $\mathbb{E}[\,\cdot\mid\mathcal{F}_{SH}]$, and invoking Lemma~\ref{lemma:local_divergence} gives the bound $\ell\cdot 4\eta G H$.
\end{proof}

\subsection{Client Availability and the Global Update}

The server aggregation at global round $t$ is denoted by $\hat{\theta}^{t} = \frac{1}{|S^t|} \sum_{j \in S^t} \theta_j^{t-1}$. The following inequality relates this aggregation to the survival probabilities $\{p_i\}$. It is the only point at which the participation model enters the analysis, and it is established directly from the general variable-size weights~\eqref{eq:general_p}.

\begin{lemma}[\textbf{Sample-to-Model Inequality}]
\label{lemma:s2m}
At every global round $t$,
\begin{equation}\label{GlobalToLocal}
    \mathbb{E}\big[\|\hat{\theta}^{t} - \theta^*\|^2\mid \mathcal{F}_{t-1}\big]
    \le
    \sum_{i \in [N]} p_i\, \|\theta_i^{t-1} - \theta^*\|^2 .
\end{equation}
\end{lemma}
\vspace{-10pt}
\begin{proof}
Since $\|\cdot\|^2$ is convex and $\hat\theta^t-\theta^*$ is an average of the $|S^t|$ vectors $\{\theta_i^{t-1}-\theta^*\}_{i\in S^t}$, Jensen's inequality gives
\begin{align}\nonumber
\mathbb{E}\big[\|\hat{\theta}^{t} - \theta^*\|^2\mid \mathcal{F}_{t-1}\big]
&=  \mathbb{E}_{S^{t}}\Bigg[ \Big\| \tfrac{1}{|S^t|} \sum_{i \in S^{t}} \theta_i^{t-1} - \theta^* \Big\|^2  \,\Big|\, \mathcal{F}_{t-1}\Bigg] \\
&\leq  \mathbb{E}_{S^{t}} \Bigg[ \tfrac{1}{|S^t|} \sum_{i \in S^{t}} \|\theta_i^{t-1} - \theta^*\|^2 \,\Big|\, \mathcal{F}_{t-1}\Bigg]. \nonumber
\end{align}
Expanding the expectation over the sampling distribution and exchanging the order of the two finite sums,
\begin{align}\nonumber
&\mathbb{E}\big[\|\hat{\theta}^{t} - \theta^*\|^2\mid \mathcal{F}_{t-1}\big] \\ \nonumber &
\leq  \sum_{S} \mathbb{P}(S^{t} = S)\, \tfrac{1}{|S|} \sum_{i \in S} \|\theta_i^{t-1} - \theta^*\|^2 \\ \nonumber &
 = \sum_{S} \mathbb{P}(S^{t}=S)\, \tfrac{1}{|S|} \sum_{i \in [N]} \mathds{1}\{i\in S\}\, \|\theta_i^{t-1} - \theta^*\|^2\\ \nonumber &
=  \sum_{i \in [N]} \Bigg[\sum_{S}\tfrac{1}{|S|}\,\mathbb{P}(S^{t}=S)\,\mathds{1}\{i \in S\}\Bigg]  \|\theta_i^{t-1} - \theta^*\|^2\\ \nonumber
&
=  \sum_{i \in [N]} p_i\, \|\theta_i^{t-1} - \theta^*\|^2 ,
\end{align}
where the last equality is the definition~\eqref{eq:general_p} of $p_i$. Since $S^t$ is independent of $\mathcal{F}_{t-1}$ and the iterates $\theta_i^{t-1}$ are $\mathcal{F}_{t-1}$-measurable, the interchange above is justified.
\end{proof}
\vspace{-8pt}

\subsection{The Risk-Neutral Case}

We now state and prove convergence of agnostic FedAvg on the induced objective~\eqref{BatchFL}.

\begin{theorem}[\textbf{Convergence of Agnostic FedAvg}]
\label{thm:main}
Let Assumptions~\ref{assump:convex} and~\ref{assump:gradient_bound} hold. Then, with step size $\eta = \Theta\big(1/\sqrt{TH}\big)$, projected agnostic FedAvg satisfies
\[
\boxed{\;
\mathbb{E}\!\left[f\Big( \frac{1}{T} \sum_{s=1}^T \hat{\theta}^{sH} \Big) - f(\theta^*)\right] = \mathcal{O}\!\left(\frac{1}{\sqrt{T}}\right).\;}
\]
\end{theorem}
\vspace{-6pt}
\begin{proof}
Fix a global round $s\in\{1,\dots,T\}$. Summing the bound of Lemma~\ref{lemma:onestep} over the $H$ local steps of round $s$ for client $i$, taking total expectations, and using $\theta_i^{(s-1)H} = \hat{\theta}^{(s-1)H}$,
\begin{multline}\label{eq:unroll}
\mathbb{E}\big[ \|\theta_i^{sH-1} - \theta^*\|^2 \big]
\leq \mathbb{E}\big[ \|\hat{\theta}^{(s-1)H} - \theta^*\|^2 \big] \\
- 2\eta \sum_{h=0}^{H-1}\mathbb{E}\big[ f_i(\theta_i^{(s-1)H+h}) - f_i(\theta^*) \big]
\\+ 2H\eta^2 (\sigma^2 + G^2).
\end{multline}
Every local iterate $\theta_i^{(s-1)H+h}$ lies in the round-$s$ window, and $\hat{\theta}^{sH}$ is an average of end-of-round iterates of that same window; hence Corollary~\ref{corollary:lipschitz_difference} applies to each pair and gives, in conditional expectation,
\begin{multline*}   
-\,\mathbb{E}\big[f_i(\theta_i^{(s-1)H+h})\,\big|\,\mathcal{F}_{(s-1)H}\big] \leq\\ -\,\mathbb{E}\big[f_i(\hat{\theta}^{sH})\,\big|\,\mathcal{F}_{(s-1)H}\big] + 4\ell\eta G H .\nonumber
\end{multline*}
Substituting into~\eqref{eq:unroll} and taking total expectations,
\begin{multline}\label{eq:unroll2}
\mathbb{E}\big[ \|\theta_i^{sH-1} - \theta^*\|^2 \big]
\leq \mathbb{E}\big[ \|\hat{\theta}^{(s-1)H} - \theta^*\|^2 \big] \\
- 2\eta H\, \mathbb{E}\big[ f_i(\hat{\theta}^{sH}) - f_i(\theta^*) \big]
\\+ 2H\eta^2 (\sigma^2 + G^2) + 8\ell\eta^2 G H^2 .
\end{multline}
Lemma~\ref{lemma:s2m} at $t=sH$, together with the tower property, gives
\begin{equation}\label{GTL2}
\mathbb{E}\big[\|\hat{\theta}^{sH} - \theta^*\|^2\big] \leq \sum_{i\in[N]} p_i\, \mathbb{E}\big[\| \theta_i^{sH-1} - \theta^* \|^2 \big] .
\end{equation}
Weighting~\eqref{eq:unroll2} by $p_i$, summing over $i\in[N]$, using $\sum_i p_i f_i = f$ from~\eqref{EmpricialRiskNeutralFL}, and combining with~\eqref{GTL2},
\begin{multline}
\mathbb{E}\big[ \|\hat{\theta}^{sH} - \theta^*\|^2 \big]
\leq \mathbb{E}\big[ \|\hat{\theta}^{(s-1)H} - \theta^*\|^2 \big] \\
- 2\eta H\, \mathbb{E}\big[ f(\hat{\theta}^{sH}) - f(\theta^*) \big]
\\+ 2H\eta^2 (\sigma^2 + G^2) + 8\ell\eta^2 G H^2 .
\end{multline}
Dividing by $2\eta H$ and rearranging,
\begin{multline}
\mathbb{E}\big[ f(\hat{\theta}^{sH}) - f(\theta^*) \big] \leq
\frac{\mathbb{E}\big[\|\hat{\theta}^{(s-1)H} - \theta^*\|^2\big]-\mathbb{E}\big[\|\hat{\theta}^{sH} - \theta^*\|^2\big]}{2\eta H} \\
+ \eta(\sigma^2 + G^2) + 4\ell\eta G H .
\end{multline}
Summing over $s = 1,\dots,T$ telescopes the first term and discards a nonnegative quantity; dividing by $T$ and applying Jensen's inequality to the convex function $f$ gives
\begin{multline}\label{eq:rate}
\mathbb{E}\!\left[ f\!\Big( \tfrac{1}{T} \sum_{s = 1}^T \hat{\theta}^{s H} \Big) - f(\theta^*) \right]
\\\leq \frac{\|\hat{\theta}^{0} - \theta^*\|^2}{2\eta H T} +  \eta (\sigma^2 + G^2) + 4 \ell \eta G H .
\end{multline}
Setting $D^2:=\|\hat{\theta}^{0} - \theta^*\|^2$ and $\eta = c/\sqrt{TH}$ for a constant $c>0$,
\begin{multline}
\mathbb{E}\!\left[ f\!\Big( \tfrac{1}{T}\textstyle\sum_{s=1}^T \hat{\theta}^{sH} \Big) - f(\theta^*) \right]
\\\leq \frac{D^2}{2c\sqrt{TH}} + \frac{c(\sigma^2 + G^2)}{\sqrt{TH}} + 4c\,\ell G \sqrt{\tfrac{H}{T}} ,
\end{multline}
and for fixed $H$ the dominant term is of order $1/\sqrt{T}$, which is the claim.
\end{proof}

Theorem~\ref{thm:main} establishes an $\mathcal{O}(1/\sqrt{T})$ rate in the number of communication rounds and is, to the best of our knowledge, the first convergence guarantee for convex federated optimization under \emph{unknown} client availability, here in the general variable-size model~\eqref{eq:general_p}. As the proof shows, the participation law enters only through Lemma~\ref{lemma:s2m}, which was established directly from~\eqref{eq:general_p}; the argument therefore covers variable-size random access without modification. Faster rates, of order $1/T$, are available for strongly convex objectives under known participation distributions, but are not applicable in the present setting. We also note that the bound~\eqref{eq:rate} contains no residual term that grows with $T$, in contrast with several earlier analyses of local-update methods. Section~\ref{sec:experiments} examines this behavior empirically under both uniform and skewed availability.

\vspace{-4pt}
\subsection{Extension to \textsc{FedeRage}}

We now transfer the guarantee to the mean--CVaR objective $\fag$ of~\eqref{eq:fag}. Define
\[
\kappa(\alpha,\gamma):= (1-\gamma)+\frac{\gamma}{\alpha},\qquad \gamma\in[0,1],\ \alpha\in(0,1],
\]
and note that $\kappa(\alpha,0)=\kappa(1,\gamma)=1$ and $\kappa(\alpha,\gamma)\ge 1$ throughout.

\begin{theorem}[\textbf{Convergence of \textsc{FedeRage}}]
\label{thm:federage}
Let Assumptions~\ref{assump:convex} and~\ref{assump:gradient_bound} hold, and let $\mathcal{C}\subseteq\Theta$ and $\mathcal{B}\subset\mathbb{R}$ be convex and compact, with projections $\Pi_\mathcal{C}$ and $\Pi_\mathcal{B}$. Consider projected agnostic FedAvg applied to $\fag$ over $\mathcal{C}\times\mathcal{B}$, with global checkpoints $(\hat{\theta}^{sH},\hat{\beta}^{sH})$ as in Algorithm~\ref{alg:federage}, and let $(\theta^*,\beta^*)\in\operatorname*{argmin}_{(\theta,\beta)\in\mathcal{C}\times\mathcal{B}} \fag(\theta,\beta)$. Then, with $\eta=\Theta\big(1/(\kappa(\alpha,\gamma)\sqrt{TH})\big)$,
\begin{multline}
    \mathbb{E}\!\left[
\fag\Big(\tfrac{1}{T}\textstyle\sum_{s=1}^T\hat{\theta}^{sH},\ \tfrac{1}{T}\textstyle\sum_{s=1}^T\hat{\beta}^{sH}\Big)
- \fag(\theta^*,\beta^*)
\right]
\\= \mathcal{O}\!\left(\frac{(1-\gamma)+\gamma/\alpha}{\sqrt{T}}\right).
\end{multline}
\end{theorem}
\vspace{-6pt}
\begin{proof}
Write $u=(\theta,\beta)\in\mathcal{C}\times\mathcal{B}$ and
\[
g_{\alpha,\gamma}(\theta, \beta;\xi)=(1-\gamma)f(\theta;\xi)+\gamma\Big[\beta+\tfrac{1}{\alpha}\big(f(\theta;\xi)-\beta\big)_+\Big],
\]
so that $\fag(u)=\mathbb{E}_{\xi\sim\mathcal{Q}^b}[g_{\alpha,\gamma}(u;\xi)]$ by~\eqref{eq:fag}.

\emph{(i) Convexity.} The map $(\theta,\beta)\mapsto f(\theta;\xi)-\beta$ is jointly convex by Assumption~\ref{assump:convex}, and $(\cdot)_+$ is convex and nondecreasing, so the composition is jointly convex; adding the affine term $\beta$ and the convex term $(1-\gamma)f(\theta;\xi)$ with nonnegative coefficients preserves convexity. Hence $g_{\alpha,\gamma}(\cdot\,;\xi)$ is convex on $\mathcal{C}\times\mathcal{B}$ for every $\xi$, and so is $\fag$ as a convex combination of such functions.

\emph{(ii) Subgradients and the algorithm.} A measurable subgradient selection of $g_{\alpha,\gamma}$ is
\[
\partial_\theta g_{\alpha,\gamma} = w\,\nabla_\theta f(\theta;\xi),\qquad
\partial_\beta G_{\alpha,\gamma} = 1-w,
\]
with $w=\omega(\theta, \beta; \xi)\in[1-\gamma,\kappa]$, which is exactly the pair of directions used in Lines~8--9 of Algorithm~\ref{alg:federage}. Consequently the local step of Algorithm~\ref{alg:federage} is a projected stochastic subgradient step on $\fag$ over $\mathcal{C}\times\mathcal{B}$, conditionally unbiased by Assumption~\ref{assump:gradient_bound} and the independence of $\xi_i^t$ from $\mathcal{F}_{t-1}$.

\emph{(iii) Constants.} Since $0\le w\le\kappa$ and $|1-w|\le\kappa$, the Lipschitz, second-moment, and variance constants of Assumption~\ref{assump:gradient_bound} and Lemma~\ref{lemma:variance} transfer to $\fag$ as
\[
\ell_{\alpha,\gamma}\le\kappa\,\ell+1,\quad
G^2_{\alpha,\gamma}\le \kappa^2(G^2+1),\quad
\sigma^2_{\alpha,\gamma}\le\kappa^2\sigma^2 ,
\]
the additive constants arising from the $\beta$-coordinate, whose subgradient is bounded by $\kappa$ independently of the data. In particular $\ell_{\alpha,\gamma}=\mathcal{O}(\kappa)$, $G_{\alpha,\gamma}=\mathcal{O}(\kappa)$, and $\sigma^2_{\alpha,\gamma}=\mathcal{O}(\kappa^2)$.

\emph{(iv) Transfer of the analysis.} Lemmata~\ref{lemma:onestep}--\ref{lemma:s2m} use only convexity, non-expansiveness of the projection onto a convex compact set, and the three constants above; moreover the server averages the pairs $(\theta_i,\beta_i)$ jointly with the same weights $1/|S^t|$, so Lemma~\ref{lemma:s2m} applies verbatim to the product iterate on $\mathcal{C}\times\mathcal{B}$, whose projection is the Cartesian product $\Pi_\mathcal{C}\times\Pi_\mathcal{B}$. Repeating the proof of Theorem~\ref{thm:main} with $(\ell,G,\sigma^2)$ replaced by $(\ell_{\alpha,\gamma},G_{\alpha,\gamma},\sigma^2_{\alpha,\gamma})$ therefore yields, for every $\eta>0$,
\begin{align*}
&\mathbb{E}\big[ \fag(\bar{\theta},\bar{\beta})-\fag(\theta^*,\beta^*) \big] \\
&\qquad\le
\frac{D^2}{2\eta TH} + \eta\big(\sigma^2_{\alpha,\gamma}+G_{\alpha,\gamma}^2\big) + 4\ell_{\alpha,\gamma}\,\eta\, G_{\alpha,\gamma} H ,
\end{align*}
where $(\bar{\theta},\bar{\beta})$ denotes the average of the $T$ global checkpoints and $D:=\sup_{u,v\in\mathcal{C}\times\mathcal{B}}\|u-v\|<\infty$. Minimizing the right-hand side over $\eta$ gives $\eta^\star=\Theta\big(1/(\kappa\sqrt{TH})\big)$ and, for fixed $H$, the stated bound of order $\kappa(\alpha,\gamma)/\sqrt{T}$.
\end{proof}

\begin{remark}[\textbf{The Price of Risk Aversion}]\label{rem:price}
Theorem~\ref{thm:federage} makes the cost of robustness explicit. Setting $\gamma=0$ gives $\kappa=1$ and recovers Theorem~\ref{thm:main} exactly, whereas $\gamma=1$ with $\alpha\to0^+$ inflates the bound by $1/\alpha$, consistent with the degeneration of the mean--CVaR objective to an essential supremum in that limit. Moderate values, $\gamma\in[0.1,0.5]$ and $\alpha\in[0.01,0.1]$, suffice to realize the empirical gains reported in Section~\ref{sec:experiments}. We note also that $\kappa$ multiplies an upper \emph{bound} and does not by itself imply slower observed convergence, since the risk-averse objective simultaneously reshapes the optimization landscape.
\end{remark}

\begin{figure*}[t]
    \centering
    \begin{minipage}[b]{0.32\textwidth}
        \centering
        \includegraphics[width=\textwidth]{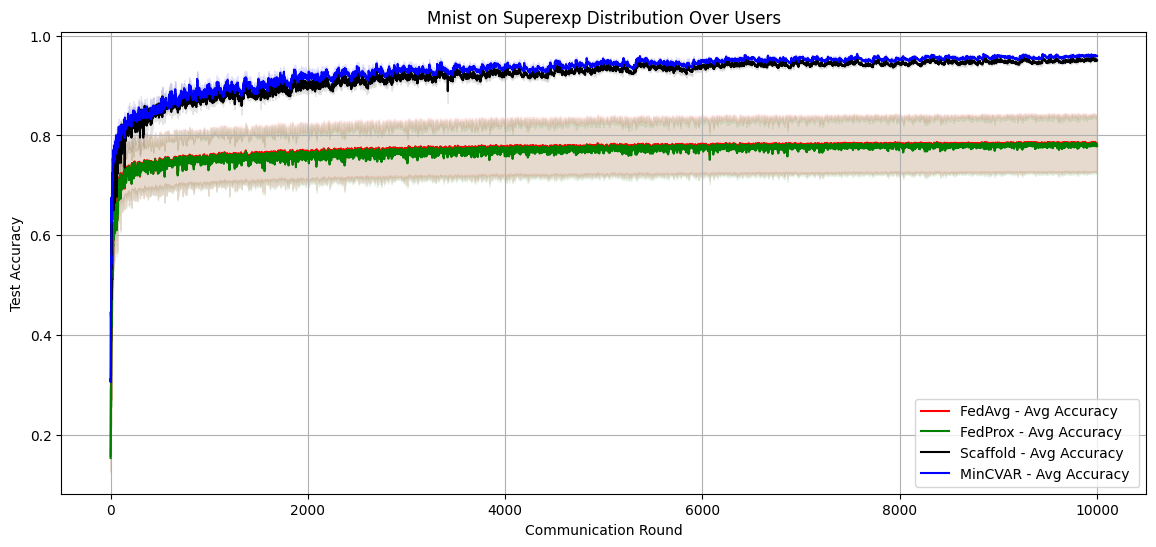}
        \caption*{(a) MNIST (non-uniform)}
    \end{minipage}
    \begin{minipage}[b]{0.32\textwidth}
        \centering
        \includegraphics[width=\textwidth]{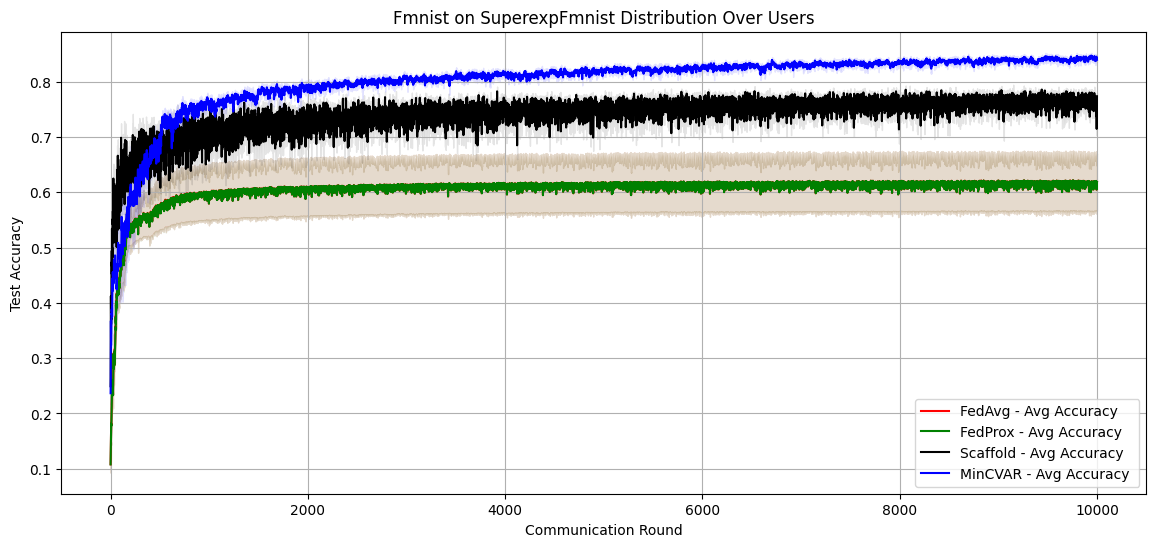}
        \caption*{(b) FashionMNIST (non-uniform)}
    \end{minipage}
    \begin{minipage}[b]{0.32\textwidth}
        \centering
        \includegraphics[width=\textwidth]{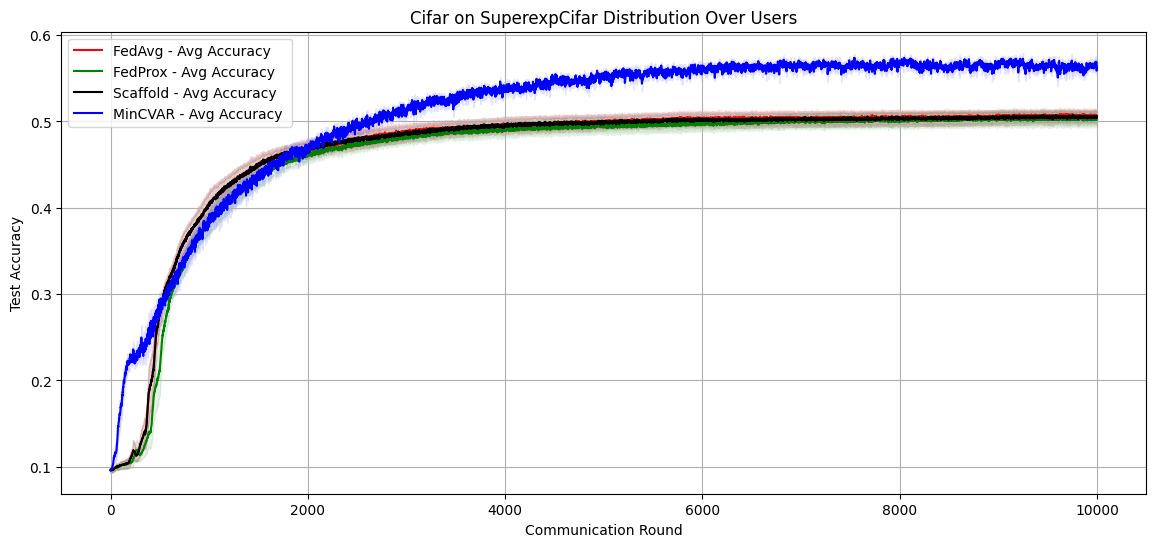}
        \caption*{(c) CIFAR-10 (non-uniform)}
    \end{minipage}
    \caption{Test accuracy over communication rounds under non-uniform availability on MNIST, FashionMNIST, and CIFAR-10.}
    \label{fig:non_uniform_results}
    \vspace{-10pt}
\end{figure*}

\section{Experiments}\label{sec:experiments}
\begin{figure}[t!]
    \centering
    \includegraphics[width=1\linewidth]{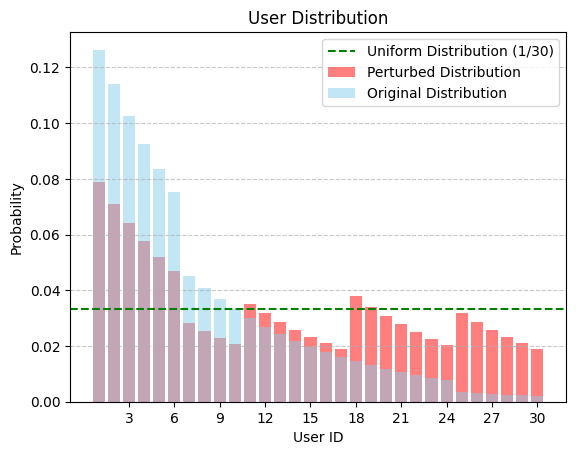}
    \caption{Participation probabilities $\{p_i\}$ for $M=3$,
    $N=30$.}
    \label{fig:user_probs}
    \vspace{-12pt}
\end{figure}
We evaluate the proposed algorithm on three standard image-classification benchmarks---\emph{MNIST}, \emph{FashionMNIST}, and \emph{CIFAR-10}---with the aim of assessing robustness to the two fundamental sources of difficulty in FL: restricted \emph{client availability} and \emph{statistical heterogeneity}.

Each dataset is partitioned across $30$ clients, with each client receiving data from at most two classes, which induces a strongly non-IID distribution and a prototypical heterogeneous FL setting. Two availability regimes are considered. Under \emph{uniform availability}, every client is equally likely to be selected in each round, as is commonly assumed in the FL literature. Under \emph{skewed availability}, participation is concentrated: most clients are selected in the majority of rounds, while a small subset---here the three clients holding the rarest classes---participate only sporadically, with per-round probabilities as low as those reported in Fig.~\ref{fig:user_probs}. This models deployments in which availability is sparse and imbalanced. In all runs the RAM selects $M=3$ of the $N=30$ clients per round.

We compare \textsc{FedeRage} against plain FedAvg and two widely adopted heterogeneity-aware baselines. \textsc{FedProx}~\cite{FedProx} augments the local objectives with a proximal term that mitigates client drift and stabilizes convergence. SCAFFOLD~\cite{SCAFFOLD} employs control variates, i.e., auxiliary parameters that correct client updates and improve the alignment of local and global gradients. Transport-based aggregation~\cite{RahimiKalogerias2026FedAVOT} is not included as a baseline, as it requires server-side knowledge of the availability law and is therefore inapplicable in the agnostic regime studied here.

In the \emph{uniform} regime, in which statistical heterogeneity is the dominant source of variance, \textsc{FedeRage} performs on par with or better than SCAFFOLD. On the most demanding benchmark, CIFAR-10, it outperforms SCAFFOLD notwithstanding the latter's more elaborate update mechanism and larger memory footprint. On MNIST and FashionMNIST, SCAFFOLD attains a marginally higher final mean accuracy, but with markedly larger dispersion across runs, so that \textsc{FedeRage} is preferable once mean and dispersion are read jointly (Table~\ref{tab:results} and Fig.~\ref{fig:non_uniform_results}); this indicates more stable and more predictable convergence. \textsc{FedeRage} also reaches competitive accuracy within fewer communication rounds.

In the more realistic \emph{skewed} regime, in which client drift and selection bias act simultaneously, the advantage of \textsc{FedeRage} is more pronounced. It outperforms all baselines on the three benchmarks, with a margin that widens on the more complex datasets (Table~\ref{tab:results} and Fig.~\ref{fig:non_uniform_results}). As shown in Fig.~\ref{fig:combined_results}, the accuracy attained for the least frequently selected clients remains high, which is an essential property when availability is irregular and unpredictable.

We further examine \emph{fairness} by evaluating the global model on each client's local test data over the final ten communication rounds. Fig.~\ref{fig:combined_results} indicates that SCAFFOLD, while adequate on average, underperforms substantially on the three least active clients, whereas \textsc{FedeRage} delivers uniformly high accuracy across the client population---a property of interest in fairness-sensitive deployments.

With respect to \emph{convergence}, Theorem~\ref{thm:federage} guarantees a rate of order $1/\sqrt{T}$ in the number of communication rounds, up to the risk-dependent factor $\kappa(\alpha,\gamma)$ discussed in Remark~\ref{rem:price}. Empirically, we observe a consistent acceleration across datasets, in that the risk-weighted updates of Algorithm~\ref{alg:federage} reach a given accuracy level in fewer rounds (Fig.~\ref{fig:non_uniform_results}).

With respect to \emph{computation and communication}, SCAFFOLD maintains and transmits a control variate of the same dimension $d'$ as the model, approximately doubling the per-round communication, whereas \textsc{FedeRage} introduces a single scalar $\beta$ per client. Beyond the reduction in memory and bandwidth, this also reduces what each client discloses, as discussed in Section~\ref{ProposedApproach}. Under equal communication budgets, the saving may be reallocated to larger models or more frequent updates at no additional cost.

In summary, the experiments indicate that \textsc{FedeRage} attains (i) higher accuracy, (ii) improved per-client fairness, (iii) faster convergence, and (iv) lower overhead across the scenarios considered, which, together with the guarantees of Section~\ref{Convergence}, supports its practicality under restricted and unknown participation.

\begin{figure*}[!t]
  \hspace{-2pt}
    \includegraphics[height=2.85in]{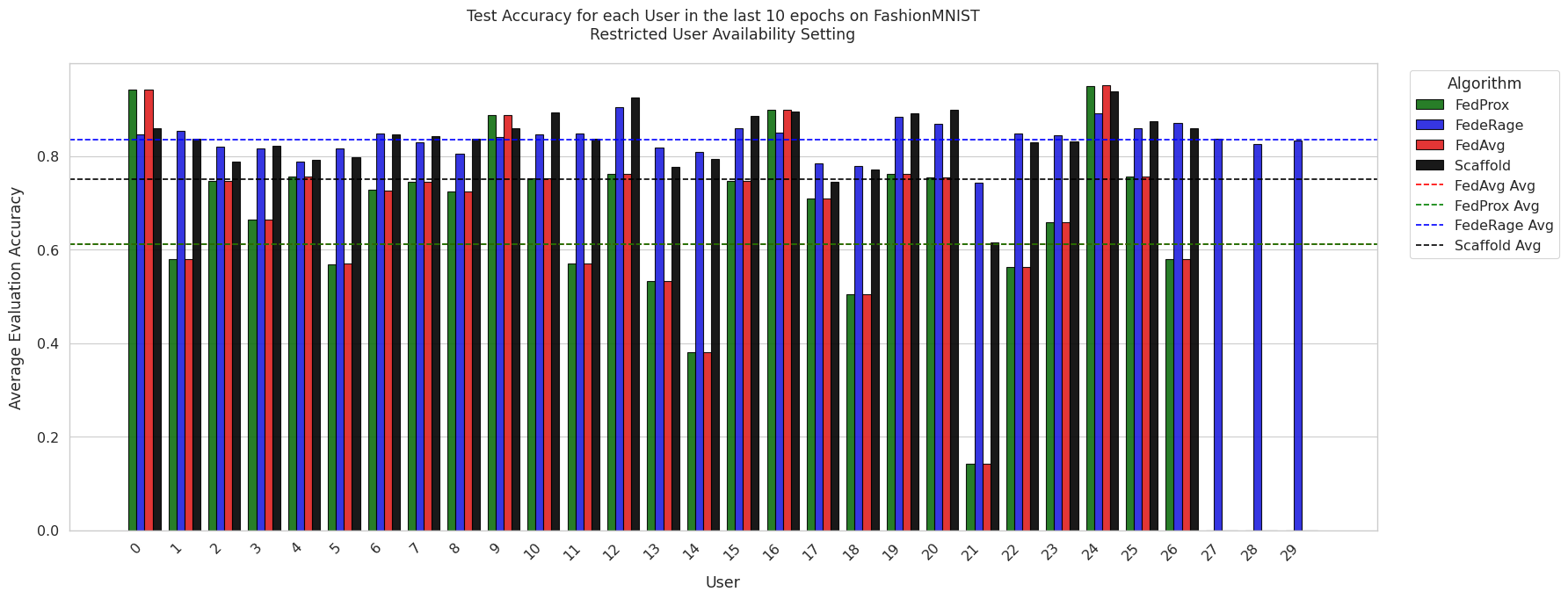}
  \vspace{-12pt}
  \caption{Restricted client availability on FashionMNIST (dashed lines mark each mean across clients for each algorithm). Under the
    non-uniform participation profile of \ref{fig:user_probs},
    \textsc{FedeRage} remains accurate even for the three most infrequently
    available clients (IDs~27--29), which the baselines effectively never
    sample.}\label{fig:combined_results}
  \vspace{-12pt}
\end{figure*}

\vspace{-4pt}
\subsection{Implementation Details and Ablations}\label{sec:impl}

All algorithms were tuned by hyperparameter sweeps on a dedicated validation set, and all reported results are averaged over five independent runs with distinct random seeds. The sweeps cover the step sizes $\eta_\theta$ and $\eta_\beta$, the number of local steps $H$, the proximal coefficient of \textsc{FedProx}, and the risk parameters $(\alpha,\gamma)$ of \textsc{FedeRage}. The selected values, the model architectures and data partitions used for each benchmark, the sensitivity of \textsc{FedeRage} to $(\alpha,\gamma)$, and per-client accuracies for the uniform-availability regime are reported in the accompanying supplementary material, together with the code required to reproduce every figure and table.

\section{Conclusion and Future Work}\label{Conclusion}

We have introduced a framework for federated learning under restricted, imbalanced, and unknown client availability. We first characterized and analyzed agnostic FedAvg under a general variable-size random-access model, establishing an $\mathcal{O}(1/\sqrt{T})$ rate for convex, possibly nonsmooth losses; we then constructed on that basis a distributionally robust, CVaR-based algorithm, \textsc{FedeRage}, and proved a matching $\mathcal{O}(\kappa(\alpha,\gamma)/\sqrt{T})$ guarantee in which the price of risk aversion is explicit. Experiments on standard benchmarks, under both uniform and highly imbalanced availability, indicate improved accuracy and stability relative to state-of-the-art methods, with the largest gains where participation is most skewed.

\vspace{0.5em}
\noindent\textbf{Future Directions---}Those are as follows:
\begin{itemize}[leftmargin=.35cm]
    \item \emph{Adaptive risk parameters.} Selecting or adapting $(\alpha,\gamma)$ automatically during training could improve robustness and generalization further.
    \item \emph{Multi-CVaR formulations.} Imposing several CVaR constraints/components, each regulating a different aspect of the objective, would afford finer control over client-level fairness and performance.
    \item \emph{Hybrid alignment and hedging.} When the availability law is partially observable, transport-based alignment and the CVaR-based hedging developed here compose naturally: one may align against an \emph{estimated} availability distribution and rely on a risk-averse local objective to absorb the residual estimation error. Characterizing the resulting bias--robustness trade-off is an appealing direction.
    \item \emph{Broader evaluation.} Extending the evaluation to further domains, such as natural language processing or deployed federated systems, would strengthen the empirical case.
    \item \emph{Fairness in FL.} The experiments show that fairness, modeled here through the minimum per-client accuracy, is improved; a formal treatment, together with a comparison to the minimax notion of~\cite{MohriAFL}, remains open.
    \item \emph{Nonconvex regimes.} Extending both the agnostic and the risk-averse analyses beyond convexity, e.g., to stationarity guarantees for deep models, is an important next step.
\end{itemize}

\bibliographystyle{IEEEtran}
\bibliography{references}

@INPROCEEDINGS{FedProx,
  author       = {Li, Tian and Sahu, Anit Kumar and Zaheer, Manzil and
                  Sanjabi, Maziar and Talwalkar, Ameet and Smith, Virginia},
  title        = {Federated Optimization in Heterogeneous Networks},
  booktitle    = {Proc. Mach. Learn. Syst. (MLSys)},
  volume       = {2},
  pages        = {429--450},
  year         = {2020},
  url          = {https://proceedings.mlsys.org/paper_files/paper/2020/hash/1f5fe83998a09396ebe6477d9475ba0c-Abstract.html}
}

@INPROCEEDINGS{MohriAFL,
  author       = {Mohri, Mehryar and Sivek, Gary and Suresh, Ananda Theertha},
  title        = {Agnostic Federated Learning},
  booktitle    = {Proc. Int. Conf. Mach. Learn. (ICML)},
  series       = {PMLR},
  volume       = {97},
  pages        = {4615--4625},
  year         = {2019},
  url          = {https://proceedings.mlr.press/v97/mohri19a.html}
}

@ARTICLE{RockafellarUryasev2000,
  author       = {Rockafellar, R. Tyrrell and Uryasev, Stanislav},
  title        = {Optimization of Conditional Value-at-Risk},
  journal      = {Journal of Risk},
  volume       = {2},
  number       = {3},
  pages        = {21--41},
  year         = {2000}
}

@INPROCEEDINGS{ReddiAdaptiveFO,
  author       = {Reddi, Sashank J. and Charles, Zachary and Zaheer, Manzil and
                  Garrett, Zachary and Rush, Keith and Kone{\v{c}}n{\'y}, Jakub and
                  Kumar, Sanjiv and McMahan, H. Brendan},
  title        = {Adaptive Federated Optimization},
  booktitle    = {Proc. Int. Conf. Learn. Representations (ICLR)},
  year         = {2021},
  url          = {https://openreview.net/forum?id=LkFG3lB13U5}
}

@ARTICLE{Kairouz2021,
  author       = {Kairouz, Peter and McMahan, H. Brendan and Avent, Brendan and others},
  title        = {Advances and Open Problems in Federated Learning},
  journal      = {Foundations and Trends in Machine Learning},
  volume       = {14},
  number       = {1--2},
  pages        = {1--210},
  year         = {2021},
  doi          = {10.1561/2200000083}
}

@INPROCEEDINGS{ChoWangJoshi2022,
  author       = {Cho, Yae Jee and Wang, Jianyu and Joshi, Gauri},
  title        = {Towards Understanding Biased Client Selection in Federated Learning},
  booktitle    = {Proc. Int. Conf. Artif. Intell. Statist. (AISTATS)},
  series       = {PMLR},
  volume       = {151},
  pages        = {10351--10375},
  year         = {2022},
  url          = {https://proceedings.mlr.press/v151/jee-cho22a.html}
}

@INPROCEEDINGS{LevyCarmonDuchiSidford2020,
  author       = {Levy, Daniel and Carmon, Yair and Duchi, John C. and Sidford, Aaron},
  title        = {Large-Scale Methods for Distributionally Robust Optimization},
  booktitle    = {Advances in Neural Information Processing Systems (NeurIPS)},
  volume       = {33},
  pages        = {8847--8860},
  year         = {2020}
}

@ARTICLE{Ribero2023,
  author       = {Ribero, M{\'o}nica and Vikalo, Haris and de Veciana, Gustavo},
  title        = {Federated Learning Under Intermittent Client Availability and
                  Time-Varying Communication Constraints},
  journal      = {IEEE Journal of Selected Topics in Signal Processing},
  volume       = {17},
  number       = {1},
  pages        = {98--111},
  year         = {2023},
  doi          = {10.1109/JSTSP.2022.3224590}
}

@article{FedNonIID,
  author  = {Zhao, Y. and Li, M. and Lai, L. and Suda, N. and Civin, D. and Chandra, V.},
  title   = {Federated Learning with Non-IID Data},
  journal = {arXiv preprint arXiv:1806.00582},
  year    = {2018},
  url     = {https://arxiv.org/abs/1806.00582}
}

@book{shapiro2009lectures,
  author    = {Shapiro, Alexander and Dentcheva, Darinka and Ruszczy{\'n}ski, Andrzej},
  title     = {Lectures on Stochastic Programming: Modeling and Theory},
  publisher = {SIAM and Mathematical Optimization Society},
  address   = {Philadelphia, PA},
  year      = {2009},
  doi       = {10.1137/1.9780898718751}
}

@inproceedings{SCAFFOLD,
  author    = {Karimireddy, Sai Praneeth and Kale, Satyen and Mohri, Mehryar and Reddi, Sashank J. and Stich, Sebastian U. and Suresh, Ananda Theertha},
  title     = {SCAFFOLD: Stochastic Controlled Averaging for Federated Learning},
  booktitle = {Proceedings of the International Conference on Machine Learning (ICML)},
  volume    = {119},
  pages     = {5132--5143},
  year      = {2020},
  url       = {https://proceedings.mlr.press/v119/karimireddy20a.html}
}

@inproceedings{FedNova,
  author    = {Wang, Jianyu and Liu, Qiang and Liang, Han and Joshi, Gauri and Poor, H. Vincent},
  title     = {Tackling the Objective Inconsistency Problem in Heterogeneous Federated Optimization},
  booktitle = {Advances in Neural Information Processing Systems (NeurIPS)},
  volume    = {33},
  pages     = {7611--7623},
  year      = {2020},
  url       = {https://proceedings.neurips.cc/paper_files/paper/2020/file/564127c03caaab942e503ee6f810f54d-Paper.pdf}
}

@inproceedings{AnarchicFL,
  author    = {Yang, Hong and Zhang, Xiaoxi and Khanduri, Poojan and Liu, Jia},
  title     = {Anarchic Federated Learning},
  booktitle = {Proceedings of the International Conference on Machine Learning (ICML)},
  volume    = {162},
  pages     = {25331--25363},
  year      = {2022},
  url       = {https://proceedings.mlr.press/v162/yang22r.html}
}

@article{Pillutla_2023,
  author  = {Pillutla, Krishna and Laguel, Youssef and Malick, J{\'e}r{\^o}me and Harchaoui, Zaid},
  title   = {Federated Learning with Superquantile Aggregation for Heterogeneous Data},
  journal = {Machine Learning},
  volume  = {113},
  number  = {5},
  pages   = {2955--3022},
  year    = {2023},
  doi     = {10.1007/s10994-023-06332-x},
  url     = {https://doi.org/10.1007/s10994-023-06332-x}
}

@inproceedings{FedDisco,
  author    = {Ye, Rui and Xu, Ming and Wang, Jianyu and Xu, Chao and Chen, Shixiang and Wang, Yong},
  title     = {FedDisco: Federated Learning with Discrepancy-Aware Collaboration},
  booktitle = {Proceedings of the International Conference on Machine Learning (ICML)},
  volume    = {202},
  pages     = {39879--39902},
  year      = {2023},
  url       = {https://proceedings.mlr.press/v202/ye23f.html}
}

@inproceedings{theodoropoulos2023ram,
  author    = {Theodoropoulos, Pantelis and Nikolakakis, Konstantinos E. and
               Kalogerias, Dionysis},
  title     = {Federated Learning under Restricted User Availability},
  booktitle = {Proc. IEEE Int. Conf. Acoust., Speech, Signal Process. (ICASSP)},
  year      = {2024},
  note      = {arXiv:2309.14176}
}

@inproceedings{RahimiKalogerias2025Agnostic,
  author    = {Rahimi, Herlock and Kalogerias, Dionysis},
  title     = {Convergence of Agnostic Federated Averaging},
  booktitle = {Proc. IEEE Int. Workshop Comput. Adv. Multi-Sensor Adaptive
               Process. (CAMSAP)},
  year      = {2025},
  pages     = {1--5}
}

@inproceedings{RahimiKalogerias2026FedAVOT,
  author    = {Rahimi, Herlock and Kalogerias, Dionysis},
  title     = {{FedAVOT}: Exact Distribution Alignment in Federated Learning
               via Masked Optimal Transport},
  booktitle = {Proc. IEEE Int. Conf. Acoust., Speech, Signal Process. (ICASSP)},
  year      = {2026},
  note      = {To appear}
}

@article{FieldGuideFedOpt,
  author  = {Wang, Jianyu and Liu, Qiang and Liang, Han and Joshi, Gauri and Poor, H. Vincent and Sahu, Anit Kumar and Stich, Sebastian and Wang, Tianyi and Kairouz, Peter and Suresh, Ananda Theertha and McMahan, H. Brendan},
  title   = {A Field Guide to Federated Optimization},
  journal = {arXiv preprint arXiv:2107.06917},
  year    = {2021},
  url     = {https://arxiv.org/abs/2107.06917}
}

@inproceedings{SemiCyclicSGD,
  author    = {Eichner, Hubert and Koren, Tomer and McMahan, Brendan and Srebro, Nathan and Talwar, Kunal},
  title     = {Semi-Cyclic Stochastic Gradient Descent},
  booktitle = {Proceedings of the International Conference on Machine Learning (ICML)},
  volume    = {97},
  pages     = {1764--1773},
  year      = {2019},
  url       = {https://proceedings.mlr.press/v97/eichner19a.html}
}

@inproceedings{FlexibleFLParticipation,
  author    = {Ruan, Yichen and Zhang, Xiaoxi and Liang, Shu-Che and Joe-Wong, Carlee},
  title     = {Towards Flexible Device Participation in Federated Learning},
  booktitle = {Proceedings of the International Conference on Artificial Intelligence and Statistics (AISTATS)},
  volume    = {130},
  pages     = {3403--3411},
  year      = {2021},
  url       = {https://proceedings.mlr.press/v130/ruan21a.html}
}

@inproceedings{FedAvg,
  author    = {McMahan, Brendan and Moore, Eider and Ramage, Daniel and Hampson, Seth and Arcas, Blaise Ag{\"u}era y},
  title     = {Communication-Efficient Learning of Deep Networks from Decentralized Data},
  booktitle = {Proceedings of the International Conference on Artificial Intelligence and Statistics (AISTATS)},
  volume    = {54},
  pages     = {1273--1282},
  year      = {2017},
  url       = {http://proceedings.mlr.press/v54/mcmahan17a.html}
}

@inproceedings{FedOPT,
  author    = {Zhao, Yuchen and Jiang, Rui and Zhou, Hao and Chen, Shuo and Xu, Xin},
  title     = {FedOPT: Federated Optimization under Strongly Convex Loss},
  booktitle = {Proceedings of the IEEE Conference on Computer Communications (INFOCOM)},
  pages     = {2916--2924},
  year      = {2022},
  url       = {https://ieeexplore.ieee.org/document/9410239}
}

@article{Bagdasaryan2018,
  author  = {Bagdasaryan, Eugene and Veit, Andreas and Hua, Yiqing and Estrin, Deborah and Shmatikov, Vitaly},
  title   = {How to Backdoor Federated Learning},
  journal = {arXiv preprint arXiv:1807.00459},
  year    = {2018},
  url     = {https://arxiv.org/abs/1807.00459}
}

@phdthesis{Coppola2015,
  author  = {Coppola, Giovanni Francesco},
  title   = {Iterative Parameter Mixing for Distributed Large-Margin Training of Structured Predictors for Natural Language Processing},
  school  = {University of Adelaide},
  year    = {2015},
  url     = {https://digital.library.adelaide.edu.au/dspace/handle/2440/108087}
}

@article{Geyer2017,
  author  = {Geyer, Robin C. and Klein, Tobias and Nabi, Moin},
  title   = {Differentially Private Federated Learning: A Client Level Perspective},
  journal = {arXiv preprint arXiv:1712.07557},
  year    = {2017},
  url     = {https://arxiv.org/abs/1712.07557}
}

@inproceedings{Hitaj2017,
  author    = {Hitaj, Briland and Ateniese, Giuseppe and P{\'e}rez-Cruz, Fernando},
  title     = {Deep Models Under the {GAN}: Information Leakage from Collaborative Deep Learning},
  booktitle = {Proceedings of the ACM SIGSAC Conference on Computer and Communications Security},
  year      = {2017},
  url       = {https://dl.acm.org/doi/10.1145/3133956.3134012}
}

@inproceedings{Hong2018,
  author    = {Hong, Mingyi and Razaviyayn, Meisam and Lee, Jong-Shi Pang},
  title     = {Gradient Primal-Dual Algorithm Converges to Second-Order Stationary Solution for Nonconvex Distributed Optimization over Networks},
  booktitle = {Proceedings of the International Conference on Machine Learning (ICML)},
  volume    = {80},
  year      = {2018},
  url       = {https://proceedings.mlr.press/v80/hong18a.html}
}

@phdthesis{Jakovetic2013,
  author  = {Jakoveti{\'c}, Du{\v s}an},
  title   = {Distributed Optimization: Algorithms and Convergence Rates},
  school  = {Carnegie Mellon University},
  year    = {2013},
  url     = {https://repository.cmu.edu/dissertations/312/}
}

@article{Khaled2019,
  author  = {Khaled, Ahmed and Mishchenko, Konstantin and Richt{\'a}rik, Peter},
  title   = {First Analysis of Local {GD} on Heterogeneous Data},
  journal = {arXiv preprint arXiv:1909.04715},
  year    = {2019},
  url     = {https://arxiv.org/abs/1909.04715}
}

@article{Konecny2017FL,
  author  = {Kone{\v c}n{\'y}, Jakub and McMahan, H. Brendan and Yu, Felix X. and Richt{\'a}rik, Peter and Suresh, Ananda Theertha and Bacon, David},
  title   = {Federated Learning: Strategies for Improving Communication Efficiency},
  journal = {arXiv preprint arXiv:1610.05492},
  year    = {2017},
  url     = {https://arxiv.org/abs/1610.05492}
}

@article{Konecny2015Opt,
  author  = {Kone{\v c}n{\'y}, Jakub and McMahan, H. Brendan and Ramage, Daniel},
  title   = {Federated Optimization: Distributed Optimization Beyond the Datacenter},
  journal = {arXiv preprint arXiv:1511.03575},
  year    = {2015},
  url     = {https://arxiv.org/abs/1511.03575}
}

@article{LeCun1998,
  author  = {LeCun, Yann and Bottou, L{\'e}on and Bengio, Yoshua and Haffner, Patrick},
  title   = {Gradient-Based Learning Applied to Document Recognition},
  journal = {Proceedings of the IEEE},
  volume  = {86},
  number  = {11},
  pages   = {2278--2324},
  year    = {1998},
  url     = {https://ieeexplore.ieee.org/document/726791}
}

@inproceedings{Li2014a,
  author    = {Li, Mu and Andersen, David G. and Park, Jun Woo and Smola, Alexander J. and Ahmed, Amr and Josifovski, Vanja and Long, James and Shekita, Eugene J. and Su, Bor-Yiing},
  title     = {Scaling Distributed Machine Learning with the Parameter Server},
  booktitle = {Proceedings of the USENIX Symposium on Operating Systems Design and Implementation (OSDI)},
  pages     = {583--598},
  year      = {2014},
  url       = {https://www.usenix.org/conference/osdi14/technical-sessions/presentation/li_mu}
}

@article{Li2019,
  author  = {Li, Tian and Sahu, Anit Kumar and Talwalkar, Ameet and Smith, Virginia},
  title   = {Federated Learning: Challenges, Methods, and Future Directions},
  journal = {arXiv preprint arXiv:1908.07873},
  year    = {2019},
  url     = {https://arxiv.org/abs/1908.07873}
}

@article{Lin2017,
  author  = {Lin, Shu-Bin and Guo, Xing and Zhou, Ding-Xuan},
  title   = {Distributed Learning with Regularized Least Squares},
  journal = {Journal of Machine Learning Research},
  volume  = {18},
  number  = {1},
  pages   = {3202--3232},
  year    = {2017},
  url     = {https://jmlr.org/papers/v18/16-577.html}
}

@article{Reddi2016,
  author  = {Reddi, Sashank J. and Kone{\v c}n{\'y}, Jakub and Richt{\'a}rik, Peter and P{\'o}cz{\'o}s, Barnab{\'a}s and Smola, Alexander},
  title   = {{AIDE}: Fast and Communication Efficient Distributed Optimization},
  journal = {arXiv preprint arXiv:1608.06879},
  year    = {2016},
  url     = {https://arxiv.org/abs/1608.06879}
}

@article{Nguyen2022,
  author  = {Nguyen, Thanh Anh and Nguyen, Tien Duc and Le, Lam Thanh and Dinh, Cuong Thanh},
  title   = {On the Generalization of Wasserstein Robust Federated Learning},
  journal = {arXiv preprint arXiv:2206.01432},
  year    = {2022},
  url     = {https://arxiv.org/abs/2206.01432}
}

@article{Yu2023,
  author  = {Yu, Yuyang and Karimireddy, Sai Praneeth and Ma, Yi and Jordan, Michael I.},
  title   = {{Scaff-PD}: Communication Efficient Fair and Robust Federated Learning},
  journal = {arXiv preprint arXiv:2307.13381},
  year    = {2023},
  url     = {https://arxiv.org/abs/2307.13381}
}

@inproceedings{Deng2020,
  author    = {Deng, Yucheng and Kamani, Mohammad Mahdi},
  title     = {Distributionally Robust Federated Averaging},
  booktitle = {Advances in Neural Information Processing Systems (NeurIPS)},
  year      = {2020},
  url       = {https://proceedings.neurips.cc/paper/2020/file/ac450d10e166657ec8f93a1b65ca1b14-Paper.pdf}
}

@article{Hong2021,
  author  = {Hong, Jiawei and Wang, Hui and Wang, Zhen and Zhou, Jie},
  title   = {Federated Robustness Propagation: Sharing Adversarial Robustness in Federated Learning},
  journal = {arXiv preprint arXiv:2106.11264},
  year    = {2021},
  url     = {https://arxiv.org/abs/2106.11264}
}

@article{Shi2023,
  author  = {Shi, Shuo and Guo, Yuchen and Wang, Di and Zhu, Yan},
  title   = {Distributionally Robust Federated Learning for Network Traffic Classification with Noisy Labels},
  journal = {IEEE Transactions on Network and Service Management},
  year    = {2023},
  url     = {https://ieeexplore.ieee.org/document/10265143}
}

\end{document}